\documentclass{article}
\usepackage[left=1 in ,right=1 in,top=1 in,bottom=1 in]{geometry}
\usepackage[pdftex]{graphicx}
\usepackage{amsfonts}
\usepackage{latexsym}
\usepackage{tabularx}
\usepackage{fancyhdr}
\usepackage{verbatim}
\usepackage{multirow}
\usepackage{framed}
\usepackage{algorithmic}
\usepackage{algorithm}
\usepackage{amsmath,amsthm,amssymb}
\usepackage[pdftex,colorlinks,citecolor=blue,linkcolor=red]{hyperref}
\usepackage{titling} 
\usepackage{graphicx, subfigure}
\usepackage{caption}
\usepackage{enumitem}
\usepackage{lscape}
\usepackage[]{authblk}

\usepackage{xcolor}
\definecolor{darkblue}{RGB}{42,34,146}

\allowdisplaybreaks[4] 

\usepackage{natbib}
\hypersetup{pdfauthor={Y.W. Park and D. Klabjan}} \hypersetup{pdftitle= Subset Selection for Multiple Linear Regression via Optimization} \hypersetup{pdfsubject=Subset Selection for Multiple Linear Regression via Optimization}

\theoremstyle{definition}

\newtheorem{lemma}{Lemma}
\newtheorem{assumption}{Assumption}

\newtheorem{proposition}{Proposition}
\theoremstyle{definition} 
\theoremstyle{definition}

\makeatletter
\newcommand{\rmnum}[1]{\romannumeral #1}
\newcommand{\Rmnum}[1]{\expandafter\@slowromancap\romannumeral #1@}
\makeatother

\title{Subset Selection for Multiple Linear Regression via Optimization } 

\author[1]{Young Woong Park \thanks{ywpark@iastate.edu}}
\author[2]{Diego Klabjan \thanks{d-klabjan@northwestern.edu}}
\affil[1]{Ivy College of Business, Iowa State University, Ames, IA, USA}
\affil[2]{Department of Industrial Engineering and Management Sciences, Northwestern University, Evanston, IL, USA}

\begin{document}

\maketitle

\begin{abstract}
Subset selection in multiple linear regression aims to choose a subset of candidate explanatory variables that tradeoff fitting error (explanatory power) and model complexity (number of variables selected). We build mathematical programming models for regression subset selection based on mean square and absolute errors, and minimal-redundancy-maximal-relevance criteria. The proposed models are tested using a linear-program-based branch-and-bound algorithm with tailored valid inequalities and big M values and are compared against the algorithms in the literature. For high dimensional cases, an iterative heuristic algorithm is proposed based on the mathematical programming models and a core set concept, and a randomized version of the algorithm is derived to guarantee convergence to the global optimum. From the computational experiments, we find that our models quickly find a quality solution while the rest of the time is spent to prove optimality; the iterative algorithms find solutions in a relatively short time and are competitive compared to state-of-the-art algorithms; using ad-hoc big M values is not recommended.
\end{abstract}





\section{Introduction}
\label{REG_section_introduction}

The multiple linear regression problem is a statistical methodology for predicting values of response (dependent) variables from a set of multiple explanatory (independent) variables by investigating the linear relationships among the variables. Given a fixed set of explanatory variables, the coefficients of the multiple linear regression model are estimated by minimizing the fitting error, where the standard setting uses the sum of squared errors ($SSE$) for measuring the fitting error. The subset selection problem, also referred to as variable selection or model selection, for multiple linear regression is to choose a subset of explanatory variables to build an efficient linear regression model. In detail, given a dataset with $n$ observations and $m$ explanatory variables, a subset of explanatory variables are used to build a regression model, where the goal is to decrease $p$, the number of explanatory variables in the model, as much as possible while maintaining error loss relatively small.

For selecting a subset of explanatory variables, an objective function is defined to measure the efficiency of the model \cite{Miller:1984}; the objective function is typically defined based on balancing the number of explanatory variables used and the fitting error. Criteria such as the mean square error ($MSE$), mean absolute error ($MAE$), adjusted $r^2$, Mallow's $C_p$, etc, are in this category for multiple linear regression and there are several works studying the $L_0$-norm-based feature selection in non-regression context \cite{Bradley-etal:98,Rinaldi:10, Western-etal:03}. There also exist objective functions balancing the magnitudes of the regression coefficients and the fitting errors; instead of the number of explanatory variables (non-zero coefficients), the regression coefficients are directly penalized. Among many variants in this category, ridge \cite{Hoerl-Kennard:70} and least absolute shrinkage and selection operator (LASSO) \cite{Tibshirani:1996} regressions are the most popular models in multiple linear regression and there also exist recent papers studying the $L_1$-norm-based feature selection in a non-regression context \cite{Fung:04, Hwang-etal}. There also exist objective functions to select variables based on mutual information gain instead of minimizing fitting error. One of the popular criteria in this category is minimum-redundancy-maximum-relevance (mRMR) proposed by \citet{DingPeng:05} and \citet{peng-etal:05}. Among various objective functions for selecting a subset, we focus on $MAE$, $MSE$, and mRMR in this paper.

Given a subset of explanatory variables, if $SSE$ is minimized, an explicit formula is available for obtaining the optimal coefficients. On the other hand, when minimizing the sum of absolute errors ($SAE$), there is no explicit formula available. For this case, a linear program (LP) \cite{Charnes-etal:55,Wagner:1959} or iterative reweighted least squares algorithm \cite{Schlossmacher:1973} can be used to build the regression model.

When subset selection is required, algorithms for optimizing $MSE$ have already been extensively studied. Among them, stepwise-type algorithms are frequently used in practice due to their computational simplicity and efficiency. An exact algorithm is to enumerate all possible regression models, but the computational cost is excessive. To overcome this computational difficulty, \citet{Furnival-Wilson:74} proposed a branch-and-bound algorithm, called \textit{leaps-and-bound}, to find the best subset for $MSE$ without enumerating all possible subsets. \citet{Miyashiro15} proposed a mathematical programming model to maximize adjust $r^2$, which is equivalent to minimizing $MSE$. Given a fixed $p$, \citet{Bertsimas-etal:15} and \citet{Bertsimas16} minimize $SSE$ using mixed integer program (MIP)-based algorithms. For subset selection of least absolute deviation regression, \citet{Konno-Yamamoto:09} presented an MIP to optimize $SAE$ given fixed $p$. \citet{Bertsimas-etal:15} proposed an MIP based algorithm for optimizing $SSE$ and $SAE$ given fixed $p$. A discrete first order method is proposed and used to warmstart the MIP formulation, which is formulated based on specially ordered sets \cite{Bertsimas-Weismantel:2005}, to avoid the use of big M. \citet{Bertsimas16} proposed an MIP based algorithm for minimizing penalized SSE given fixed $p$. For a detailed review of algorithms for subset selection, the reader is referred to \citet{Miller:2002}. 

Selecting a subset of explanatory variables with non-zero regression coefficients can be compared to general optimization problems with cardinality constraints. While several early works (e.g., \citet{Bienstock:96, Farias-Nemhauser:03}) study general optimization problems with cardinality constraints from the optimization theoretical point of view, recent work directly focuses on mathematical programming models for regression subset selection. The MIP models in \citet{Bertsimas-Shioda:09}, \citet{Bertsimas-etal:15}, and \citet{Konno-Yamamoto:09} assume fixed $p$ and the cardinality constraint is explicit in the models. Their models are distinguished by the objective functions and how they formulate subset selection; \citet{Konno-Yamamoto:09} optimized $SAE$ by introducing binary variables, \citet{Bertsimas-etal:15} optimized $SSE$ or $SAE$, \citet{Bertsimas-Shioda:09} optimized $SSE$ without introducing binary variables. In contrast to the models in \citet{Bertsimas-Shioda:09}, \citet{Bertsimas-etal:15}, and \citet{Konno-Yamamoto:09}, we optimize $MSE$ and $MAE$ without fixing $p$. \citet{Miyashiro15} proposed a mathematical programming model to maximize adjust $r^2$, which is equivalent to minimize $MSE$. To the best of authors' knowledge, the model in \citet{Miyashiro15} is the only mathematical programming model that directly maximizes adjust $r^2$ (equivalent to minimizing $MSE$), but there is no mathematical programming model directly optimizing mRMR or $MAE$.

Basic multiple linear regression analyses require a data matrix with $n > m + 1$; i.e., the number of observations must be greater than the number of explanatory variables plus one. Otherwise, the $n-1$ linearly independent explanatory variables and one intercept variable yield a regression model with zero fitting error given a full rank data matrix. However, in practice, it is not that uncommon to have a data set with $m \geq n-1$. For example, gene information has many attributes (explanatory variables) while only a few observations are usually available. In statistics, subset selection when $m \geq n$ is called \textit{high dimensional variable selection}. Note that, if each row of the data matrix is an observation, the length of the data matrix is greater than the width when $m<n$, and the width of the data matrix is greater when $m \geq n$. Based on the shape of the data matrix, we hereafter refer to the cases $m < n$ and $m \geq n$ as \textit{thin} and \textit{fat cases}, respectively. For the fat case, \citet{Stodden:2006} studied how model selection algorithms behave with a different but fixed ratio of $\frac{m}{n}$. \citet{Candes-Tao:07} proposed an $l_1$-regularized problem based approach, called \textit{Dantzig selector}. However, their approach does not explicitly take into account the number of selected variables, which is different from our models.

Our contributions are as follows.
\begin{enumerate}
\item  We present mathematical programs for the subset selection problem that directly minimize the popular criteria $MAE$ and $MSE$. To the best of our knowledge, the proposed model for $MAE$ is the first mathematical programming formulation that directly optimizes $MAE$. The proposed model for $MSE$ is an equivalent model to the model of \citet{Miyashiro15} which optimizes a different objective function; our work has been conducted simultaneously with \citet{Miyashiro15}. In the computational experiment, we observe that the proposed models quickly return a good candidate solution when solved by a commercial optimization solver.
\item We propose the first mathematical programming formulation that directly optimizes mRMR for the thin case, which also can be used for the fat case with trivial modifications. A modified version of the model is also proposed to balance mRMR and the fitting errors. The modified version integrates the mRMR-based feature selection and regression model building steps to obtain a model considering both mRMR and the error-based objective $MAE$ or $MSE$. The computational experiment shows that the proposed models return different subsets from the $MSE$, $MAE$, and mRMR models in a relatively short computational time.
\item For the proposed mathematical programs for $MSE$ and $MAE$, we propose exact and heuristic approaches to obtain big M values. The performances of the models with different big M values are discussed in the computational experiment. Further, the performance of the proposed big M-based formulations are compared with alternative mathematical programming formulations and implementations. 
\item To overcome computational difficulties of the MIP models, we propose an iterative algorithm that gives a quality solution in a relatively short computational time for the fat case. We show that the algorithm yields a local optimal and we propose a randomized version of the algorithm to guarantee convergence to the global optimum. The computational experiment shows that the proposed algorithms are competitive compared to the state-of-the-art benchmarks.
\end{enumerate}
The structure of the paper is as follows. In \autoref{REG_section_formulation_subset_selection}, the mathematical models for the thin case with $MAE$, $MSE$, and mRMR objectives are derived. In \autoref{section_many_variables}, for the fat case, we propose the iterative algorithm based on the mathematical models and derive the randomized version of the algorithm with the convergence result. Finally, we present computational experiments in \autoref{REG_section_computational_experiment}.




\section{Mathematical Models for Thin Case \texorpdfstring{$(m < n)$}{(m < n)}}
\label{REG_section_formulation_subset_selection}

In this section, we derive mathematical programs to directly optimize $MAE$, $MSE$, and mRMR for the thin case. Throughout this paper, the following notation is used: 
\begin{enumerate}[noitemsep]
\item[] $n$ : number of observations
\item[] $m$ : number of explanatory variables 
\item[] $p$ : number of selected explanatory variables 
\item[] $I = \{1,\cdots,n\}$: index set of observations
\item[] $J = \{1,\cdots,m\}$: index set of explanatory variables
\item[] $a=[a_{ij}] \in \mathbb{R}^{n \times m}$: data matrix corresponding to the independent variables
\item[] $a_j \in \mathbb{R}^n$: independent variable $j \in J$
\item[] $b=[b_i] \in \mathbb{R}^{n}$: data vector corresponding to the dependent variable.
\item[] $\rho_{jk}$: absolute sample correlation between explanatory variables $j,k \in J$
\item[] $\rho_j$: absolute sample correlation between explanatory variable $j \in J$ and the dependent variable
\end{enumerate}
For all mathematical models derived, the following decision variables are used:
\begin{enumerate}[noitemsep]
\item[] $x_j$: coefficient of the $j^{th}$ explanatory variable, $j \in J$
\item[] $y$: intercept of the regression model 
\item[] $t_i$: error term of the $i^{th}$ observation, $i \in I$
\item[] $z_j= \left \{
	\begin{array}{ll}
		1  &  $ if explanatory variable $x_j$ is included in the model $\\
		0 & $ otherwise $
	\end{array} , \quad j \in J
\right.$. 
\end{enumerate}

Note that the multiple linear regression model takes the form $b_i = y+ \sum_{j \in J} a_{ij}x_j + t_i$, for $i \in I$. Let us consider a regression model with fixed subset $\hat{S}$ of $J$. For the minimization of $SAE$ given $\hat{S}$, the following LP gives optimal regression coefficients:
\begin{equation}
\label{formulation_SAE_opt}
\min \sum_{i \in I} \bar{t}_i \quad \mbox{s.t.} \quad t_i = \sum_{j \in \hat{S}} a_{ij} x_j  + y - b_i, -\bar{t}_i \leq t_i \leq \bar{t}_i, \bar{t}_i \geq 0, i \in I .
\end{equation}
We later use this LP as a subroutine when we need to construct a regression model that minimizes $SAE$ given a fixed subset. Next we review the three subset selection criteria, which we use for the mathematical programming formulations. In the followings, $SSE$ and $SAE$ are taken with respect to a subset $\hat{S}$ of cardinality $p$.
\begin{enumerate}
\item $MSE$ is one of the most popular criteria \citep{Tamhane-Dunlop:1999}, defined as $\frac{SSE}{n-1-p}$. By minimizing $MSE$, we can balance $SSE$ and $p$ because $SSE$ decreases in $p$. Another popular criteria is adjusted $r^2$, defined as $r^2_{a} = 1 - \frac{MSE}{SST / (n-1)}$, where $SST$ is the total sum of squares. Because $\frac{SST}{n-1}$ is a constant, maximizing $r^2_{a}$ is equivalent to minimizing $MSE$. This explains the equivalence of our model and \citet{Miyashiro15}.
\item $MAE$, defined as $\frac{SAE}{n-1-p}$, is an alternative to $MSE$ for reducing the effect of outliers. Note that $MAE$ is defined similarly to $MSE$, where $SAE$ is used instead of $SSE$. $MAE$ is a widely used criterion that is less sensitive to outliers and can also be used as an evaluation criterion when the model is fitted using squared errors \cite{Harrell:2001}. For a detailed discussion of $MAE$ compared with $MSE$, the reader is referred to \citet{Chai:04} and \citet{Willmott:2005}
\item mRMR, defined as $\frac{1}{p} \sum_{j \in \hat{S}} \rho_j - \frac{1}{p^2} \sum_{j,k \in \hat{S}} \rho_{jk}$, is frequently used to select features prior to running statistical models. By maximizing mRMR, the highly correlated explanatory variables to the dependent variable are selected (the first term in the expression) while maintaining the variables that are far away from each other (the second term in the expression).
\end{enumerate}
We remark that the first objective is one of the most popular criteria practitioners use for selecting a subset, the second objective is a variant of the first, which is mainly concerned with reducing the effect of outliers, and the last objective is useful for screening the explanatory variables in an extreme fat case data.

\subsection{Mean Square and Absolute Errors}

In this section, we derive mathematical programs for $MAE$ and $MSE$ in Sections \ref{REG_subsection_formulation_subset_mae} and \ref{REG_subsection_formulation_subset_mse}. For the proposed models, valid values for big M, which is an upper bound for the regression coefficients, and valid inequalities are derived in Sections \ref{REG_subsection_bigM} and \ref{REG_subsection_VI}.

\subsubsection{Minimization of \texorpdfstring{$MAE$}{Minimization of MAE}}
\label{REG_subsection_formulation_subset_mae}

Observe that $MAE = \frac{SAE}{n-1-p}$ has two terms ($SAE$ and $p$) that can be written as $SAE = \sum_{i \in I} |t_i|$ and $p = \sum_{j \in J} z_j$ in terms of the decision variables. Using these expressions, we can write a mathematical model
\begin{subequations}
\label{formulation_subset_mae_derive1}
\begin{align}
\min \quad & \textstyle \frac{\sum_{i \in I} |t_i|}{n-1 -\sum_{j \in J} z_j} \label{formulation_subset_mae_derive1_a} \\[-3pt]
s.t.\quad& \textstyle t_i = \sum_{j \in J} a_{ij} x_j  + y - b_i, & i \in I , \label{formulation_subset_mae_derive1_b} \\[-3pt]
& \textstyle -M z_j \leq x_j \leq M z_j, & j \in J, \label{formulation_subset_mae_derive1_c}\\[-3pt]
& \textstyle z_j \in \{0,1\}, t,x,y \mbox{ unconstrained}. \label{formulation_subset_mae_derive1_d}
\end{align}
\end{subequations}
to minimize $MAE$. Observe that, if we add constraint $\sum_{j \in J} z_j = p$ to \eqref{formulation_subset_mae_derive1} given fixed $p$, we obtain an easier problem, which is equivalent to the model presented by \citet{Konno-Yamamoto:09} since the denominator of the objective becomes constant. By adding cardinality constraint with fixed $p$ and by replacing \eqref{formulation_subset_mae_derive1_c} with specially order sets based constraints, we obtain the model presented in \citet{Bertsimas-etal:15}. The remaining development is completely different from the work in \citet{Konno-Yamamoto:09} or \citet{Bertsimas-etal:15} and thus new. This is due to the fact that they assume fixed $p$ which implies that model \eqref{formulation_subset_mae_derive1} is already linear. In our case we have to linearize this model which is not a trivial task. Note that \eqref{formulation_subset_mae_derive1} is a Mixed Integer Linear Fraction Programming (MIFLP). There are numerous studies discussing solving MIFLP problems in the original form without linearizing the objective function, which is different from our approach linearizing the objective function to reformulate \eqref{formulation_subset_mae_derive1}. The readers are referred to \citet{schaible2004recent} and \citet{stancu2012fractional} for detailed reviews of fractional programming literature.

Note that $M$ in \eqref{formulation_subset_mae_derive1_c} is a constant, which is an upper bound for $x_j$'s, that we have not yet specified. \citet{Konno-Yamamoto:09} set an arbitrary large value for $M$ in their study. For now, let us assume that a proper value of $M$ is given (we derive a valid value for $M$ in a later section). To linearize nonlinear objective \eqref{formulation_subset_mae_derive1_a}, we introduce
\begin{equation}
\label{equation_u_definition}
u =  \frac{\sum_{i \in I} |t_i|}{n-1-\sum_{j \in J} z_j}.
\end{equation}
Observe that $u$ explicitly represents $MAE$. We linearize objective function \eqref{formulation_subset_mae_derive1_a} by adding \eqref{equation_u_definition} as a constraint and setting $u$ as the objective function. Then, \eqref{formulation_subset_mae_derive1} can be rewritten as
\begin{subequations}
\label{formulation_subset_mae_derive2}
\begin{align}
\min \quad & \textstyle u \label{formulation_subset_mae_derive2_a} \\[-3pt]
s.t.\quad & \textstyle \sum_{i \in I} |t_i| = (n-1)u - u \sum_{j \in J} z_j, \label{formulation_subset_mae_derive2_b}\\[-3pt]
& \textstyle t_i = \sum_{j \in J} a_{ij} x_j  + y - b_i, & i \in I , \label{formulation_subset_mae_derive2_c} \\[-3pt]
& \textstyle -M z_j \leq x_j \leq M z_j, & j \in J, \label{formulation_subset_mae_derive2_d}\\[-3pt]
& \textstyle u \geq 0, z_j \in \{0,1\}, t,x,y \mbox{ unconstrained}. \label{formulation_subset_mae_derive2_e}
\end{align}
\end{subequations}
In order to linearize nonlinear constraint \eqref{formulation_subset_mae_derive2_b}, we introduce $v_j = u z_j$, $j \in J$, which can be linearized using standard linearization techniques . Using a linearization technique \cite{glover1975improved} with proper settings, we obtain 
\begin{subequations}
\label{formulation_subset_mae_derive3}
\begin{align}
\min \quad & \textstyle u \label{formulation_subset_mae_derive3_a} \\[-3pt]
s.t.\quad & \textstyle \sum_{i \in I} |t_i| = (n-1)u - \sum_{j \in J} v_j \label{formulation_subset_mae_derive3_b}\\[-3pt]
& \textstyle t_i = \sum_{j \in J} a_{ji} x_j  + y - b_i, & i \in I,  \label{formulation_subset_mae_derive3_c} \\[-3pt]
& \textstyle -M z_j \leq x_j \leq M z_j, & j \in J,  \label{formulation_subset_mae_derive3_d}\\[-3pt]
& \textstyle v_j \leq u,  & j \in J,  \label{formulation_subset_mae_derive3_e}\\[-3pt]
& \textstyle u-M(1-z_j) \leq v_j \leq M z_j,  & j \in J,  \label{formulation_subset_mae_derive3_f}\\[-3pt]
& \textstyle v_j \geq 0, u \geq 0, z_j \in \{0,1\}, t,x,y \mbox{ unconstrained}. \label{formulation_subset_mae_derive3_g}
\end{align}
\end{subequations}
Observe that we use $M$ again in \eqref{formulation_subset_mae_derive3_f} and a proper value for $M$ is derived in a later section. We conclude that \eqref{formulation_subset_mae_derive3} is a valid formulation for \eqref{formulation_subset_mae_derive2} by the following proposition.

\begin{proposition}
\label{proposition_2_and_3_equivalent}
An optimal solution to model \eqref{formulation_subset_mae_derive2} and an optimal solution to model \eqref{formulation_subset_mae_derive3} have the same objective function value.
\end{proposition}
The proof is given in \autoref{appendix_proofs_of_lemmas} and is based on the fact that feasible solutions to \eqref{formulation_subset_mae_derive2} and \eqref{formulation_subset_mae_derive3} map to each other. Observe that the signs of $t,x$, and $y$ in \eqref{formulation_subset_mae_derive3} are not restricted. In order to make all variables non-negative, we introduce $x_j^+$, $x_j^-$, $y^+$ and $y^-$, in which $x_j = x_j^+ - x_j^-$ and $y = y^+ - y^-$. We also use $t_i^+$ and $t_i^-$, where $t_i = t_i^+ - t_i^-$, to replace the absolute value function in \eqref{formulation_subset_mae_derive3_b}. Finally, we obtain mixed integer program \eqref{formulation_subset_mae} for regression subset selection with the $MAE$ objective.
\begin{subequations}
\label{formulation_subset_mae}
\begin{align}
\min \quad & \textstyle u \label{formulation_subset_mae_a} \\[-3pt]
s.t.\quad & \textstyle \sum_{i=1}^n (t_i^+ + t_i^-) = (n-1)u - \sum_{j \in J} v_j, \label{formulation_subset_mae_b}\\[-3pt]
& \textstyle t_i^+ - t_i^- = \sum_{j=1}^m a_{ij} (x_j^+ - x_j^-)  + (y^+ - y^-) - b_i, &i \in I,  \label{formulation_subset_mae_c} \\[-3pt]
& \textstyle x_j^+ \leq M z_j, & j \in J,  \label{formulation_subset_mae_d}\\[-3pt]
& \textstyle x_j^- \leq M z_j, & j \in J,   \label{formulation_subset_mae_e}\\[-3pt]
& \textstyle v_j \leq u,  & j \in J ,  \label{formulation_subset_mae_f}\\[-3pt]
& \textstyle u-M(1-z_j) \leq v_j \leq M z_j,  & j \in J, \label{formulation_subset_mae_g}\\[-3pt]
& \textstyle x_j^+ \geq 0, x_j^- \geq 0, y^+ \geq 0, y^- \geq 0, v_j \geq 0, u \geq 0, t_i^+ \geq 0, t_i^- \geq 0, z_j \in \{0,1\} \label{formulation_subset_mae_h}
\end{align}
\end{subequations}

It is known that either $t_i^+$ or $t_i^-$ is equal to 0 if $\sum_{i \in I} |t_i|$ is minimized in the objective function. However, since $\sum_{i \in I} |t_i|$ is not directly minimized and binary variables are present in \eqref{formulation_subset_mae_derive3}, we give the following proposition in order to make sure that \eqref{formulation_subset_mae_derive3} is equivalent to \eqref{formulation_subset_mae}, where the proof is given in \autoref{appendix_proofs_of_lemmas}.

\begin{proposition}
\label{proposition_complementary_t}
An optimal solution to \eqref{formulation_subset_mae} must have either $t_i^+ = 0$ or $t_i^- = 0$ for every $i \in I$.
\end{proposition}

By \autoref{proposition_complementary_t}, it is easy to see that \eqref{formulation_subset_mae_b} is equivalent to \eqref{formulation_subset_mae_derive3_b}. Therefore, \eqref{formulation_subset_mae} correctly solves \eqref{formulation_subset_mae_derive1}. A final remark regarding the model is with regard to the dimension of the formulation. For a dataset with $m$ candidate explanatory variables and $n$ observations, formulation \eqref{formulation_subset_mae} has $2n+4m+3$ variables (including $m$ binary variables) and $n+5m+1$ constraints (excluding non-negativity constraints).

\subsubsection{Minimization of \texorpdfstring{$MSE$}{Minimization of MSE}}
\label{REG_subsection_formulation_subset_mse}

In this section, we derive a quadratically constrained mixed integer programming model based on the results in Section \ref{REG_subsection_formulation_subset_mae}, which gives an equivalent formulation to \citet{Miyashiro15} as  maximizing adjusted $r^2$ is equivalent to minimizing $MSE$. Our work has been conducted simultaneously with \citet{Miyashiro15}.

Observe that the only difference between $MSE$ and $MAE$ is that $MSE$ has $\sum_{i=1}^n t_i^2$, whereas $MAE$ has $\sum_{i=1}^n |t_i|$. Hence, the left hand side of \eqref{formulation_subset_mae_b} is replaced by $\sum_{i=1}^n (t_i^+ - t_i^-)^2$. Also, in order to make the constraint convex, we use inequality instead of equality. Hence, we use 
\begin{equation}
\label{formulation_subset_constraint_for_mse}
\sum_{i \in I} (t_i^+ - t_i^-)^2 \leq (n-1)u - \sum_{j \in J} v_j 
\end{equation}
instead of \eqref{formulation_subset_mae_b}. Finally, the mixed integer quadratically constrained program with the convex relaxation reads
\begin{equation}
\label{formulation_subset_mse}
\min \{ u | \eqref{formulation_subset_constraint_for_mse}, \eqref{formulation_subset_mae_c}-\eqref{formulation_subset_mae_h} \}.
\end{equation}
Note that we use inequality in \eqref{formulation_subset_constraint_for_mse} to have the convex constraint, but $u$ is correctly defined only when \eqref{formulation_subset_constraint_for_mse} is at equality. Hence, we need the following proposition.

\begin{proposition}
\label{proposition_mse_equality_at_opt}
An optimal solution to \eqref{formulation_subset_mse} must satisfy \eqref{formulation_subset_constraint_for_mse} at equality.
\end{proposition}

The proof is given in \autoref{appendix_proofs_of_lemmas}. By \autoref{proposition_mse_equality_at_opt}, we know that \eqref{formulation_subset_constraint_for_mse} is satisfied at equality at an optimal solution, hence \eqref{formulation_subset_mse} correctly solves the problem.

\subsubsection{Big M for \texorpdfstring{$x_j$}{x}'s and \texorpdfstring{$v_j$}{v}'s}
\label{REG_subsection_bigM}

Deriving a tight and valid value of $M$ in \eqref{formulation_subset_mae} and \eqref{formulation_subset_mse} is crucial for two reasons. For optimality, too small values cannot guarantee optimality even when the optimization model is solved optimally. For computation, a large value of $M$ causes numerical instability and slows down the branch-and-bound algorithm. Recall that we assume that a valid value of $M$ is given for the formulations \eqref{formulation_subset_mae} and \eqref{formulation_subset_mse} and that the same notation $M$ is used for both $x_j$'s and $v_j$'s. However, $x_j$'s and $v_j$'s are often in different magnitudes. Hence, it is necessary to derive distinct and valid values of $M$ for $x_j$'s and $v_j$'s. 

In this section, we derive valid values of $M$ for $x_j$'s and $v_j$'s in \eqref{formulation_subset_mae}. The result also holds for \eqref{formulation_subset_mse} with trivial modifications. Among the two exact approaches proposed in this section, the first approach is based on the logic similar to \citet{Bertsimas-etal:15}, where a similar approach is provided without a validity check for a different problem minimizing SSE given a fixed $p$. We also provide a computationally faster procedure for $M$ for $x_j$'s in \eqref{formulation_subset_mse} as an alternative. Both of the $M$ values do not cause any numerical problems in our experiments.

First, let us consider $M$ for $v_j$'s. Observe that a valid $M$ for $v_j$ must be greater than all possible values for $u$. However, it is generally better to have tight upper bounds. Hence, we use $mae_m$, the mean absolute error of an optimal regression model with all $m$ explanatory variables, as upper bounds. We set
\begin{equation}
\label{formula_bigM_for_v}
M := mae_m 
\end{equation}
for every $v_j$ in \eqref{formulation_subset_mae}. Note that \eqref{formula_bigM_for_v} can be calculated by LP formulation \eqref{formulation_SAE_opt} in polynomial time. By using the $M$ value in \eqref{formula_bigM_for_v}, we treat regression models that have worse objective function values than $mae_m$ as infeasible. 

Next, let us consider $M$ for $x_j$'s in \eqref{formulation_subset_mae} for $MAE$. We start with the following assumption.
\begin{assumption}
\label{assumption_no_zero_total_error}
Dataset $\{b,a_1, a_2, \cdots, a_m \}$ is linearly independent.
\end{assumption}
This assumption implies that there is no regression model with total error equal to 0 among all possible subsets of the $m$ explanatory variables. This is a mild assumption because, in practice, we typically have a dataset with structural and random noises and it is unlikely to have zero error.

In order to find a valid value of $M$ for $x_j$'s in \eqref{formulation_subset_mae}, we formulate an LP. Let $\mu$ be the decision variable having the role of $M$. Let $\bar{b} = \frac{\sum_{i \in I} b_i}{n}$ and $T_{max} = \sum_{i \in I} |b_i - \bar{b}|$ be the average of $b_i$'s and the maximum total error bound allowed, respectively. Any attractive regression model should have the total error less than $T_{max}$ in order to justify the effort of building a regression model, because $SAE > T_{max}$ with $p>0$ gives an automatically worse objective function value than the model with no explanatory variable. This requirement is written as $\sum_{i \in I} (t_i^+ + t_i^-) \leq T_{max}$. Because for now we are only concerned with feasibility, we can ignore $u$ and all related constraints and variables \eqref{formulation_subset_mae_b}, \eqref{formulation_subset_mae_f}, $z_j$'s, and $v_j$'s. Then, we have the following feasibility set:
\begin{subequations}
\label{LP_bigM}
\begin{align}
 & \textstyle \sum_{i \in I} (t_i^+ + t_i^-) \leq T_{max}, \label{LP_bigM_a}\\[-3pt]
& \textstyle t_i^+ - t_i^- = \sum_{j \in J} a_{ij} (x_j^+ - x_j^-)  + (y^+ - y^-) - b_i, & i \in I,\label{LP_bigM_b} \\[-3pt]
& \textstyle x_j^+ \leq \mu, & j \in J,  \label{LP_bigM_c}\\[-3pt]
& \textstyle x_j^- \leq \mu, & j \in J, \label{LP_bigM_d}\\[-3pt]
& \textstyle \mu \geq 0, x_j^+ \geq 0, x_j^- \geq 0, y^+ \geq 0, y^- \geq 0, t_i^+ \geq 0, t_i^- \geq 0. \label{LP_bigM_e}
\end{align}
\end{subequations}
For notational convenience, let $Y = (x^+, x^-, y^+, y^-, t^+, t^-, \mu )$ be a vector in \eqref{LP_bigM}.

Next, let us try to increase $x_k^+$ to its maximum value. For a fixed $0 < \varepsilon < 1$, we define the objective as 
\begin{center}
max \quad $x_k^+ - \varepsilon \mu$.
\end{center}
With the second term, we force $\mu$ to be the maximum value we need, yet not preventing a further increment of $x_k^+$. From the linear program
\begin{equation}
\label{LP_bigM_x_k_plus}
\max \{x_k^+ - \varepsilon \mu|\mbox{\eqref{LP_bigM_a}-\eqref{LP_bigM_e}}, x_k^- = 0\},
\end{equation}
we obtain $\hat{M}_k^+$, a candidate for $M$, from the value of $\mu$ of an optimal solution solution to \eqref{LP_bigM_x_k_plus}. Similarly, $\hat{M}_k^-$ is obtained from $\max \{x_k^- - \varepsilon \mu |\mbox{\eqref{LP_bigM_a}-\eqref{LP_bigM_e}}, x_k^+ = 0\}$. Then the maximum value for explanatory variable $x_k$ can be obtained by setting $\hat{M}_k = \max\{ \hat{M}_k^+, \hat{M}_k^-\}$. Finally, considering all explanatory variables, we define $\hat{M}$ as
\begin{equation}
\label{bigM_definition}
\hat{M} = \max_{j \in J} \hat{M}_j.
\end{equation}
Before we proceed, we first need to make sure that \eqref{LP_bigM_x_k_plus} is not unbounded so that the values are well defined.

\begin{proposition}
\label{proposition_boundedLP}
Linear program \eqref{LP_bigM_x_k_plus} is bounded.
\end{proposition}

\begin{lemma}
\label{lemma_conversion}
Let $\hat{M}$ be the value obtained from \eqref{bigM_definition} and $\bar{Y} = (\bar{x}^+, \bar{x}^-, \bar{y}^+,\bar{y}^-,\bar{v}_j, \bar{u}, \bar{t}^+, \bar{t}^-, \bar{z})$ be a feasible solution of \eqref{formulation_subset_mae} with $\hat{M}$ and $SAE$ less than or equal to $T_{max}$. Then, $\tilde{Y} = (\bar{x}^+, \bar{x}^-, \bar{y}^+,\bar{y}^-, \bar{t}^+, \bar{t}^-, \hat{M})$ is a feasible solution for \eqref{LP_bigM}.
\end{lemma}

The proofs are given in \autoref{appendix_proofs_of_lemmas}. Note that \autoref{lemma_conversion} implies that \eqref{LP_bigM} covers all possible values of $x_j^+$ and $x_j^-$ of \eqref{formulation_subset_mae} with the maximum total error bound $T_{max}$. Note also that $\hat{M}$ in \eqref{bigM_definition} is the maximum value out of all possible values of $x_j^+$ and $x_j^-$ that \eqref{LP_bigM} covers.

\begin{proposition}
\label{proposition_bigM}
For all regression models with $SAE$ less than or equal to $T_{\max}$, 
$\hat{M}$ in \eqref{bigM_definition} is a valid upper bound for $x_j^+$'s and $x_j^-$'s in \eqref{formulation_subset_mae}.
\end{proposition}
\begin{proof}
For a contradiction, suppose that $\hat{M}$ is not a valid upper bound for $x_j$'s in \eqref{formulation_subset_mae}. That is, there exists a regression model $(\bar{x}^+, \bar{x}^-, \bar{y}^+, \bar{y}^-)$ with total error less than $T_{max}$ but $\bar{x}_q^+ > \hat{M}$, in which $\bar{x}_q^+$ is the coefficient for explanatory variable $q$. However, by \autoref{lemma_conversion}, we must have a corresponding feasible solution $\bar{Y} = (\bar{x}^+, \bar{x}^-, \bar{y}^+, \bar{y}^-, \bar{t}^+, \bar{t}^-, \hat{M})$ for \eqref{LP_bigM} with $\bar{x}_q^+ > \hat{M}$. Note that $\bar{Y}$ must satisfy $\bar{x}_q^+ \leq M_q^+$ from \eqref{LP_bigM_d}. Then, $M_q^+ \geq x_q^+ > \hat{M}$ implies $M_q^+ > \hat{M}$. This contradicts definition \eqref{bigM_definition}. A similar argument holds if $\bar{x}_q^- > \hat{M}$. Hence, $\hat{M}$ is a valid upper bound.
\end{proof}

Observe that a similar approach can be used to derive a valid value of $M$ for $x_j$'s in \eqref{formulation_subset_mse} for $MSE$. Calculating a valid value $M$ for $x_j$'s in \eqref{formulation_subset_mae} and \eqref{formulation_subset_mse} consists of solving $2m$ LPs and $2m$ quadratically constrained convex quadratic programs (QCP). Hence, we conclude that it can be obtained in polynomial time.

To reduce the computational time for the big $M$ calculation, we present an alternative approach that works for $MSE$ models from a different perspective.

Note that we can obtain coefficients of an optimal regression model that minimizes $SSE$ over all explanatory variables as $\hat{x} = (a^{\top} a)^{-1} a^{\top} b$, where $a \in \mathbb{R}^{n \times m}$ and $b \in \mathbb{R}^{n \times 1}$. This is equivalent to solving $Ax = B$, with $A = a^{\top} a \in \mathbb{R}^{m \times m}$ and $B = a^{\top} b \in \mathbb{R}^{m \times 1}$. For a rational number $r = \frac{r_{num}}{r_{den}}$ ($r_{num} \in \mathbb{Z}$, $r_{den} \in \mathbb{N}$, $r_{num}$ and $r_{den}$ relative prime numbers), a rational vector $B = [\beta_1,\cdots,\beta_m]$, and a rational matrix $A = [\alpha_{ij}]_{i=1,\cdots,m, j=1,\cdots,m}$, let us define 
\begin{enumerate}[noitemsep]
\item[] \textit{size}($r$) := $1 + \lceil \log_2 (|r_{num}|+1) \rceil + \lceil \log_2 (r_{den}+1) \rceil$
\item[] \textit{size}($B$) := $\sum_{i \in I} $\textit{size}($\beta_i$)
\item[] \textit{size}($A$) := $m^2 + \sum_{i \in I} \sum_{j \in J}$\textit{size}($\alpha_{ij}$).
\end{enumerate}
Note that it is known that the size of solutions to $Ax = B$ are bounded. Here, we extend this over the various submatrices of $A$ and subvectors of $B$ encountered in our subset selection procedure. The following proposition provides a valid value of $M$.
\begin{proposition}
\label{proposition_bigM_MSE_alternative}
Value $M :=  2^{size(A)size(B) -1}$ is a valid upper bound for $x_j^+$'s and $x_j^-$'s in \eqref{formulation_subset_mse}.
\end{proposition}
The proof of \autoref{proposition_bigM_MSE_alternative} and the omitted detailed derivations are available in Section 1 of the online supplement. Observe that $size(A)$ and $size(B)$ can be calculated in polynomial time. In detail, it takes $O(mnh)$ in which $h$ is the number of digits of the largest absolute number among all elements of $A$ and $B$ to compute $M$. Recall that the previous approach requires to solve $2m$ QCPs. Hence, we have an alternative polynomial time big $M$ calculation procedure which is computationally more efficient than the one provided by \autoref{proposition_bigM}. However, this procedure yields a larger value of $M$.

\subsubsection{Valid Inequalities}
\label{REG_subsection_VI}

To accelerate the computation, we apply several valid inequalities at the root node of the branch and bound algorithm. Let $u^{heur}$ and $\bar{u}$ be the objective function values of a heuristic and the LP relaxation, respectively. Let $\beta_j^0$ ($\beta_j^1$) be the objective function value of the LP relaxation of \eqref{formulation_subset_mae} after fixing $z_j = 0$ ($z_j = 1$). Then, the following inequalities are valid for \eqref{formulation_subset_mae}:
\begin{align}
& v_j \leq u^{heur} z_j, & \quad  j \in J \label{vi_cut0}\\[-3pt]
& v_j \geq \bar{u} z_j, & \quad  j \in J \label{vi_cut1}\\[-3pt]
& u \geq (\beta_j^1- \beta_j^0)z_j + \beta_j^0 & \quad  j \in J \label{vi_cut2}
\end{align}
We do not provide proofs as it is trivial to establish their validity. In \autoref{fig:examplevi}, we illustrate the valid inequalities. In both figures, the dark and light-shaded areas represent the feasible and infeasible region, respectively, after applying the valid inequalities, whereas the combined area represents the original feasible region of the formulation. In Figure \ref{fig:vi1}, valid inequalities \eqref{vi_cut0} and \eqref{vi_cut1} are presented. Value $u^*$ is the optimal objective function value. In Figure \ref{fig:vi2}, $\bar{u}$ is the objective function value of the LP relaxation with non-integer $z_j$ before applying \eqref{vi_cut2}. The black circles represent $(0,\beta_j^0)$ and $(1,\beta_j^1)$ that give valid lower bounds for any integer solution. Observe that integer feasible solutions (empty rectangles in the figure) are in the feasible region after applying the valid inequality.

\begin{figure}[ht]
\begin{center}
		\subfigure[Valid inequalities \eqref{vi_cut0} and \eqref{vi_cut1}]{%
           \includegraphics[scale=0.4]{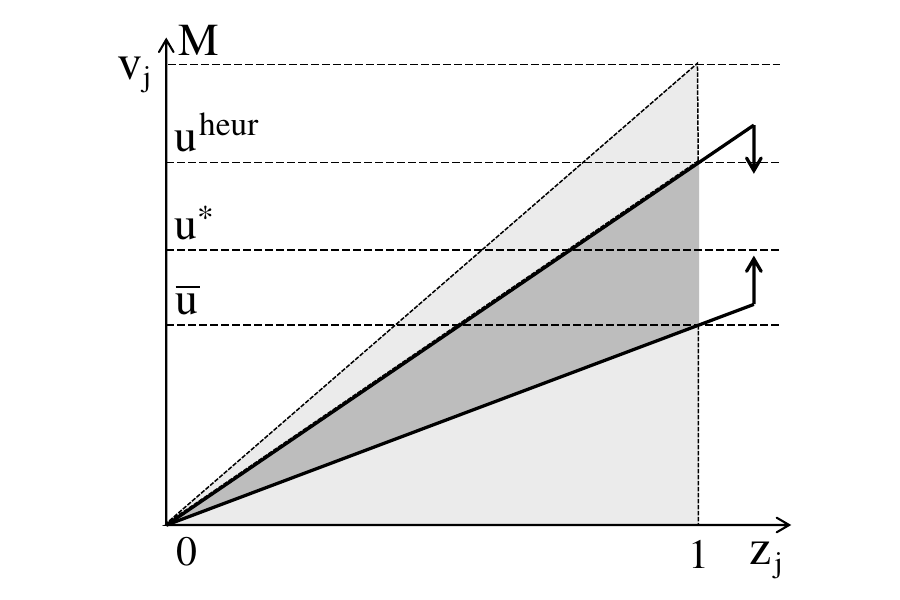} \label{fig:vi1}
        }\quad
        \subfigure[Valid inequality \eqref{vi_cut2}]{%
           \includegraphics[scale=0.4]{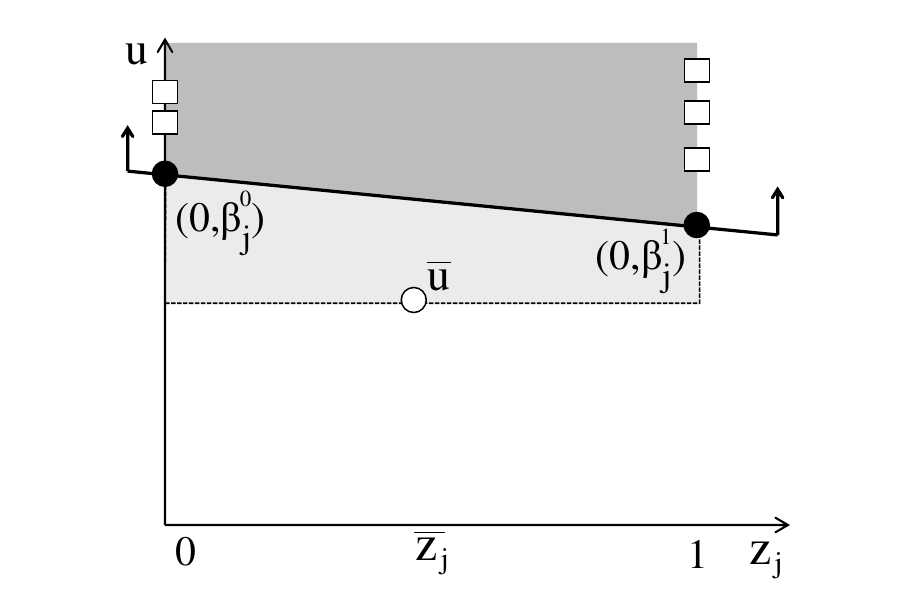} \label{fig:vi2}
        }
\end{center}      
\caption{Illustration of the valid inequalities }\label{fig:examplevi}
\end{figure}

Note that \eqref{vi_cut0} can be generated given an objective value of any feasible solution. For $MAE$, generating \eqref{vi_cut1} and \eqref{vi_cut2} requires solving one LP and two LPs, respectively, for each $j \in J$. For $MSE$, generating \eqref{vi_cut1} and \eqref{vi_cut2} requires solving one QCP and two QCPs, respectively.

\subsection{Minimal-Redundancy-Maximal-Relevance}
\label{REG_subsection_formulation_mrmr}

Given $a$ and $b$, the mRMR criterion can be modeled as the following optimization problem:
\begin{equation}
\max_{S} \frac{1}{|S|} \sum_{j \in S} \rho_j - \frac{1}{|S|^2} \sum_{j,k \in S} \rho_{jk}. \label{mrmr_def}
\end{equation}

In this section, we assume that we want to find set $S$ with $|S| = p$, which is different from the previous treatment, i.e., here $p$ is fixed. Using the binary variables $z_j$ previously defined, \eqref{mrmr_def} can be written as
\begin{equation}
\max \Big\{ \frac{p \sum_{j \in J} \rho_j z_j - \sum_{j,k \in J} \rho_{jk}z_j z_k}{p^2} \Big| z \in \{0,1\}^m, \sum_{j \in J} z_j = p \Big\}.\label{mrmr_opt}
\end{equation}
By introducing new variable $z_{jk} \equiv z_j z_k$, \eqref{mrmr_opt} can be converted into an MIP as follows.
\begin{subequations}
\label{formualtion_IG_initial}
\begin{align}
\max \quad & \displaystyle \sum_{j \in J} \frac{\rho_j}{p}  z_j - \sum_{j,k \in J} \frac{\rho_{jk}}{p^2} z_{jk}\\
s.t.\quad & z_{jk} \geq z_j + z_k - 1 & j,k \in J\\
& \sum_{j \in J} z_j = p\\
& z_j \in \{0,1\}, z_{jk} \in \{0,1\}
\end{align}
\end{subequations}
This model is the first approach in the literature that guarantees global optimality for the mRMR criterion. However, if \eqref{formualtion_IG_initial} is solved approximately, it may not improve the solution of the greedy algorithm, which is often used in practice. Further, if the selected features by \eqref{formualtion_IG_initial} will be used for building a regression model, it is beneficial to consider the regression fit simultaneously by integrating the mRMR and the traditional error-based objectives. Hence, we propose to combine \eqref{formulation_subset_mae} and \eqref{formualtion_IG_initial} to optimizing SAE while guaranteeing good objective function values for \eqref{formualtion_IG_initial}. 

Let $\bar{\Omega}$ be the optimal objective function value of \eqref{mrmr_def}. With a fractional parameter $\lambda \in [0,1]$, the following constraint guarantees at most $\frac{\lambda}{100} \%$ away from $\bar{\Omega}$:
\begin{center}
$\frac{1}{|S|} \sum_{j \in S} \rho_j - \frac{1}{|S|^2} \sum_{j,k \in S} \rho_{jk} \geq \bar{\Omega} - \mbox{sign(}\bar{\Omega}) \cdot \lambda \cdot |\bar{\Omega}| .$
\end{center}
Regardless of the sign of $\bar{\Omega} $, the lower bound is smaller than $\bar{\Omega}$ by $\lambda \cdot \bar{\Omega}$. Combining the constraint with the MIP model optimizing SAE, the following MIP is obtained.
\begin{subequations}
\label{formulation_mae_mrmr}
\begin{align}
\min \quad & \textstyle \sum_{i \in I} t_i^+ + t_i^- \label{formulation_mae_mrmr_a} \\
s.t.\quad& \textstyle t_i^+ - t_i^- = \sum_{j \in J} a_{ij} (x_j^+ - x_j^-)  + y^+ - y^- - b_i, & i \in I , \label{formulation_mae_mrmr_b} \\
& \textstyle -M z_j \leq x_j \leq M z_j, & j \in J, \label{formulation_mae_mrmr_c}\\
& \sum_{j \in J} \frac{\rho_j}{p}  z_j - \sum_{j,k \in J} \frac{\rho_{jk}}{p^2} z_{jk} \geq \bar{\Omega} - \mbox{sign(}\bar{\Omega}) \cdot \lambda \cdot |\bar{\Omega}|, \label{formulation_mae_mrmr_e}\\
& z_{jk} \geq z_j + z_k - 1, & j,k \in J,\\
& \sum_{j \in J} z_j = p,\\
& \textstyle z_j \in \{0,1\}, z_{jk} \in \{0,1\}, t_i^+, t_i^-,x_j^+, x_j^-,y^+,y^- \geq 0 \label{formulation_mae_mrmr_d}
\end{align}
\end{subequations}
Note that \eqref{formulation_mae_mrmr} combines the mRMR feature selection and regression model building procedures, whereas \eqref{formualtion_IG_initial} provides pre-screening of explanatory variables for the regression model building step in the next stage. A similar model for SSE can be obtained by replacing \eqref{formulation_mae_mrmr_a} by $\sum_{i \in I} (t_i^+ + t_i^-)^2$, and \eqref{formulation_mae_mrmr} can be used for the fat case with trivial modifications.




\section{Mathematical Models and Algorithms for Fat Case \texorpdfstring{$(m \geq n)$}{(m >= n)}}
\label{section_many_variables}

Let us consider the fat case, in which there are more explanatory variables than observations. A natural extension of \eqref{formulation_subset_mae} or \eqref{formulation_subset_mse} for the fat case is to add cardinality constraint $\sum_{j \in J} z_j \leq n-2$. This constraint successfully selects a proper number of explanatory variables in many cases, however, we found that the objective $MAE$ and $MSE$ for the fat case could be problematic in some cases; the penalty on the number of explanatory variables by $MAE$ or $MSE$ is too weak (or strong) and the optimal solution selects $n-2$ (or 0) explanatory variables.

Minimizing $SAE$ can be thought as approximating the right-hand side (dependent values $b$) using a combination of columns (explanatory variables). If we have more linearly independent explanatory variables than observations, we can always build a regression model with $SAE = 0$. Hence, if we allow $p \geq n-1$, then the $MAE$ objective is not useful. Further, due to the definition of $MAE = \frac{SAE}{n-1-p}$, we must have $p \leq n-2$ in order to make the numerator positive.

Suppose we can select $n-2$ explanatory variables out of $m$ $(m> n-2)$ candidate explanatory variables. Because $SAE$ converges to zero as we add more linearly independent explanatory variables and because $p=n-2$ and $n$ are close to each other, $SAE$ can be near zero. In this case, having $n-2$ explanatory variables might not be penalized enough by the definition of $MAE$. This could make $p=n-2$ optimal and it actually happens in many instances studied in \autoref{REG_section_computational_experiment}, which is not a desired solution in most cases. Hence, even with the restriction $p \leq n-2$, $MAE$ may not be a useful criteria. In order to fix this issue, we use a slightly modified objective function by additionally penalizing having too many explanatory variables in the regression model:
\begin{equation}
\label{def_MAE_plus}
MAE_a = \frac{SAE + \frac{p}{n-2}mae_0}{n-1-p},
\end{equation}
where $mae_0 = \frac{\sum_{i \in I}|b_i - \bar{b}|}{n-1}$ is the mean absolute error of the optimal regression model with $p=0$. Observe that \eqref{def_MAE_plus} is equivalent to $MAE$ when $p=0$. The penalty term increases as $p$ increases. To optimize \eqref{def_MAE_plus}, \eqref{formulation_subset_mae} can be modified as 
\begin{equation}
\label{formulation_MAE_more_m}
\min \{u | \eqref{formulation_subset_mae_b} - \eqref{formulation_subset_mae_e}, \eqref{formulation_subset_mae_h} , \sum_{j \in J} z_j \leq n-2, v_j \leq u + \frac{mae_0}{n-2}, u + \frac{mae_0}{n-2} - M(1-z_j) \leq v_j \leq M z_j   \}.
\end{equation}
For the detailed derivations and modifications, please consider Appendix \ref{appendix_new_obj}. Finally, we remark that all algorithms proposed in this section can also optimize $MSE$ and $MAE$.

\subsection{Core Set Algorithm}

\noindent Observe that \eqref{formulation_MAE_more_m} might be difficult to solve optimally if the data is large because the number of binary variables increases as $m$ increases. To overcome this computational difficulty and get a quality solution quickly, we develop an iterative algorithm based on \eqref{formulation_MAE_more_m} and the popular core set concept in computer science and operations research \cite{Har-Peled:2011}.

Let $C$ be a subset of $J$ such that $|C| \leq n-2$, with the cardinality of $C$ defined by
\begin{equation}
\label{definition_Theta}
\Theta = |C| = \min \{ n \theta, n-2 \},
\end{equation}
where $0 < \theta < 1$ is a fraction that defines the target cardinality of $C$. We refer to $C$ as the \textit{core set} and iteratively solve
\begin{equation}
\label{formulation_core_MAE}
\min \{u | \eqref{formulation_subset_mae_b}, \eqref{formulation_subset_mae_c} - \eqref{formulation_subset_mae_e}, \eqref{formulation_subset_mae_h}, \eqref{new_constraint_MAE_f},\eqref{new_constraint_MAE_g} \}
\end{equation}
that is obtained by dropping the cardinality constraint \eqref{constraint_cardinality} from \eqref{formulation_MAE_more_m}. Hereafter, we assume that \eqref{formulation_core_MAE} is always solved with $C$ instead of $J$, with $|C| \leq n-2$ being ensured by \eqref{definition_Theta}.

We present the algorithmic framework in \autoref{algo_excessive_m} based on the core set concept. Let $S^*$ be the current best subset in \autoref{algo_excessive_m} with corresponding objective function value $mae_a^*$. In Steps \ref{algo_excessive_m_line1} - \ref{algo_excessive_m_line3}, we initialize core set $C$ with cardinality not exceeding $\Theta$. We solve \eqref{formulation_core_MAE} with $C$ in Step \ref{algo_excessive_m_line5} and then update $C$ in Step \ref{algo_excessive_m_line6}. We iterate these steps until there is no improvement of the objective function value from a previous iteration. We remark that the worst case run time of \autoref{algo_excessive_m} is exponential because \eqref{formulation_core_MAE} is solved by the branch-and-bound algorithm, which has exponential worse case run time, in each iteration. However, in practice, \autoref{algo_excessive_m} terminates quickly as shown in the experimental results in Section \ref{REG_section_computational_experiment}.

\begin{algorithm}[ht]
\caption{Core-Heuristic}        
\label{algo_excessive_m}                           
\begin{algorithmic}[1]   
\vspace{0.1cm}
\REQUIRE $\theta$ (core set factor)\\ 
\STATE $\Theta \gets \min \{ n \theta, n-2 \}$ \label{algo_excessive_m_line1}
\STATE $(S^*,mae_a^*) \gets$ stepwise heuristic with $J$ and constraint $p \leq \Theta$ \label{algo_excessive_m_line2}
\STATE $(S^*, mae_a^*, C, \Theta ) \gets $ \textit{Update-Core-Set}($S^*, mae_a^*, \Theta$) \label{algo_excessive_m_line3}
\WHILE{objective function value is improving} \label{algo_excessive_m_line4}
\STATE \quad $(S^*,mae_a^*) \gets$ solve  \eqref{formulation_core_MAE} with $C$ \label{algo_excessive_m_line5}
\STATE \quad $(S^*, mae_a^*, C, \Theta  ) \gets $ \textit{Update-Core-Set}($S^*, mae_a^*, \Theta$)  \label{algo_excessive_m_line6}
\ENDWHILE
\end{algorithmic}
\end{algorithm}

We next explain how the core set is updated. The updating algorithm is presented in \autoref{algo_update_coreset}. In Steps \ref{algo_update_coreset_line13} and \ref{algo_update_coreset_line14}, the idea is to keep the explanatory variables of the current best subset $S^*$ in the core set and additionally selecting explanatory variables not in $S^*$ based on scores $T_j$. The score is defined based on how much of the error could be reduced if we add explanatory variable $j$ to the current best subset $S^*$. In Steps \ref{algo_update_coreset_line1} - \ref{algo_update_coreset_line6}, we calculate $T_j$'s and $E_a^j$'s by checking neighboring subsets. Note that $T_j$'s can be calculated by LP formulation \eqref{formulation_SAE_opt}. In Steps \ref{algo_update_coreset_line7} - \ref{algo_update_coreset_line12}, we update the current best subset $S^*$ if we found a better solution in Steps \ref{algo_update_coreset_line1} - \ref{algo_update_coreset_line6}. If $S^*$ is updated, we go to Step \ref{algo_update_coreset_line1} and restart the algorithm with new $S^*$ and $\Theta$. Observe that $E_a^j$'s in Steps \ref{algo_update_coreset_line1} - \ref{algo_update_coreset_line3} are only for updating $S^*$ in Step \ref{algo_update_coreset_line8}, whereas $T_j$'s and $E_a^j$'s in Steps \ref{algo_update_coreset_line4} - \ref{algo_update_coreset_line6} are also used to define $B$ in Step \ref{algo_update_coreset_line13}.

\begin{algorithm}[ht]
\caption{Update-Core-Set}        
\label{algo_update_coreset}                           
\begin{algorithmic}[1]    
\REQUIRE $S^*$ (current best subset), $mae_a^*$ (current best obj value), $\Theta$ (core set cardinality) \\
\ENSURE $S^*$ (new current best subset), $mae_a^*$ (new current best obj value), $C$ (new core set), $\Theta$ (new core set cardinality)\\
\vspace{0.1cm}
\FOR{$j \in S^*$} \label{algo_update_coreset_line1}
\STATE $T_j \gets$ SAE of subset $S^* \setminus \{j\}$, $E_a^j \gets \frac{T_j + \frac{|S^*|-1}{n-2} mae_0}{n-1-|S^*|-1}$ \label{algo_update_coreset_line2}
\ENDFOR \label{algo_update_coreset_line3}
\FOR{$j \in J \setminus S^*$} \label{algo_update_coreset_line4}
\STATE $T_j \gets$ SAE of subset $S^* \cup \{j\}$, $E_a^j \gets \frac{T_j + \frac{|S^*|+1}{n-2} mae_0}{n-1-|S^*|+1}$ \label{algo_update_coreset_line5}
\ENDFOR \label{algo_update_coreset_line6}
\STATE \textbf{if} $\min_{j \in J} E_a^j < mae_a^*$ \label{algo_update_coreset_line7}
\STATE \quad update $S^*$ to $T_j$ that gives minimum $E_a^j$ value \label{algo_update_coreset_line8}
\STATE \quad \textbf{if} $|S^*| = \Theta$ \textbf{then} $\Theta \gets \min \{ \Theta +1, n-2 \}$ \label{algo_update_coreset_line9}
\STATE \quad $mae_a^* \gets \min_{j \in J} E_a^j$ \label{algo_update_coreset_line10}
\STATE \quad  go to Step \ref{algo_update_coreset_line1} \label{algo_update_coreset_line11}
\STATE \textbf{end if} \label{algo_update_coreset_line12}
\STATE $B \gets$ $\{ \Theta -|S^*|$ explanatory variables in $J \setminus S^*$ with smallest $T_j$'s $\}$ \label{algo_update_coreset_line13}
\STATE $C \gets S^* \cup B$ \label{algo_update_coreset_line14}
\end{algorithmic}
\end{algorithm}

Let us define the neighborhood of set $\bar{S}$ as
\begin{equation}
\mathcal{N}(\bar{S}) = \{ S \subset J | |S \bigtriangleup \bar{S}| \leq 1 \}, \label{def_neighbor}
\end{equation}
where $S \bigtriangleup \bar{S}$ defines the symmetric difference of $S$ and $\bar{S}$. Through the following propositions, we show that \autoref{algo_excessive_m} does not cycle and terminates with a local optimal solution based on the neighborhood definition given in \eqref{def_neighbor}.

\begin{proposition}
\autoref{algo_excessive_m} does not cycle.
\end{proposition}
\noindent For the proof, see Lemmas OS 6 and 7 (OS stands for online supplement), which guarantee that there is no cycle in the loop of \autoref{algo_excessive_m}.
\begin{proposition}
\autoref{algo_excessive_m} gives a local optimum.
\end{proposition}
\begin{proof}
When \autoref{algo_excessive_m} terminates, all subsets that are neighbors to $S^*$, defined by \eqref{def_neighbor}, are evaluated in Steps \ref{algo_update_coreset_line1} - \ref{algo_update_coreset_line6} of \autoref{algo_update_coreset}, but there is no better solution than $S^*$. Hence, \autoref{algo_excessive_m} gives a local optimum.
\end{proof}

\subsection{Randomized Core Set Algorithm}

We also present a randomized version of \autoref{algo_excessive_m}, which we call \textit{Core-Random}. By constructing a core set randomly and by executing the while loop of \autoref{algo_excessive_m} infinitely many times, we show that we can find a global optimal solution with probability 1 when $\theta = 1$. The randomized version of \textit{Update-Core-Set} is presented in \autoref{algo_update_coreset_random}. \textit{Update-Core-Set-Random} is similar to \textit{Update-Core-Set}, with one difference. Instead of the greedy approach in Steps \ref{algo_update_coreset_line13}-\ref{algo_update_coreset_line14} of \autoref{algo_update_coreset}, we randomly choose $n-2$ explanatory variables one-by-one without replacement based on a probability distribution.

\begin{algorithm}[ht]
\caption{Update-Core-Set-Random}        
\label{algo_update_coreset_random}                           
\begin{algorithmic}[1]    
\REQUIRE $S^*$ (current best subset), $mae_a^*$ (current best obj value), $\Theta$ (core set cardinality) \\
\ENSURE $S^*$ (new current best subset), $mae_a^*$ (new current best obj value), $C$ (new core set), $\Theta$ (new core set cardinality)\\
\vspace{0.1cm}

\STATE Steps \ref{algo_update_coreset_line1} - \ref{algo_update_coreset_line12} of \autoref{algo_update_coreset} \label{algo_update_coreset_random_line1}
\STATE Define initial probabilities based on \eqref{def_selection_prob} \label{algo_update_coreset_random_line2}
\STATE $C \gets \emptyset$, $\bar{J} \gets J$ \label{algo_update_coreset_random_line3}
\STATE \textbf{while} $|C| < \Theta$ \label{algo_update_coreset_random_line4}
\STATE \quad Select explanatory variable $k$ in $\bar{J}$ based on generalized Bernoulli with probabilities $p_j$ \label{algo_update_coreset_random_line5}
\STATE \quad $C \gets C \cup \{ k \}$, $\bar{J} \gets \bar{J} \setminus \{k \}$, renormalize $p_j$'s based on \eqref{def_prob_renormalize} \label{algo_update_coreset_random_line6}
\STATE \textbf{end-while} \label{algo_update_coreset_random_line7}
\end{algorithmic}
\end{algorithm}

Let us next describe the initial probability distribution used in Step \ref{algo_update_coreset_random_line2} of \autoref{algo_update_coreset_random}. Let $U_j$ be the current best objective function value whenever explanatory variable $j$ is included in the regression model. We update $U_j$'s at each iteration throughout the entire algorithm. In detail, we set $U_j := mae_a^*$ for $j \in S^*$ whenever current best objective function value $mae_a^*$ and subset $S^*$ are updated. In order to enhance the local optimal search, we give a bonus to the columns currently in $S^*$ by setting weight $w_j = 0.5$ if $j \in S^*$ and $w_j = 1$ if $j \in J \setminus S^*$. Observe that giving the same weight for all $j \in J$ is equivalent to a random search. On the other hand, if the weight for $S^*$ is much smaller (hence much greater selection probability) than the weight for $j \in J \setminus S^*$, then we are likely to choose all variables in  $S^*$, which is similar to \autoref{algo_update_coreset}. By means of a computational experiment, we found out that giving twice more weights for $j \in J \setminus S^*$ compared to $j \in S^*$ balances exploration and exploitation.

We normalize $U_j$'s and generate $\bar{U}_j$'s so that $\min_{j \in J} \bar{U}_j = -0.5$ and $\max_{j \in J} \bar{U}_j = 0.5$. In detail,
\begin{equation}
\label{definition_u_j}
\bar{U}_j = \frac{w_j U_j - \bar{U}_{mid}}{\bar{U}_{max} - \bar{U}_{min}} \quad \mbox{ for } j \in J,
\end{equation}
where $\bar{U}_{min} = \min_{j \in J} w_j U_j$, $\bar{U}_{max} =\max_{j \in J} w_j U_j$, and $\bar{U}_{mid} = (\bar{U}_{max}-\bar{U}_{min})/2$. Finally, we define probabilities using the exponential function
\begin{equation}
q_j = \frac{e^{-\bar{U}_j}}{\sum_{j \in J} e^{-\bar{U}_j}} \quad \mbox{ for } j \in J. \label{def_selection_prob}
\end{equation}
From definitions \eqref{definition_u_j} and \eqref{def_selection_prob}, we have the following characteristic of $q_j$'s.
\begin{lemma}
\label{lemma_min_max_prob_ratio}
We have $\displaystyle \frac{ \max_{j \in J} q_j}{\min_{j \in J} q_j} \leq 2.72$ for any values of $q_j$'s.
\end{lemma}
\noindent The proof is available in Section 2 of the online supplement. By the lemma, we know that the best explanatory variables in $S^*$ has at most $2.72$ times higher chance than the worst explanatory variable to be picked. Observe that, once we select an explanatory variable in Step \ref{algo_update_coreset_random_line5}, we need to exclude the selected explanatory variable in the next selection iteration. This can be thought as sampling without replacement. Let $\bar{J}$ be the set of explanatory variables that have not been selected in the previous selection iterations. In Step \ref{algo_update_coreset_random_line6}, we add explanatory variable $k$ to the core set and exclude it from $\bar{J}$. Then, we normalize the probability distribution based on
\begin{equation}
\label{def_prob_renormalize}
q_j = \frac{q_j}{\sum_{j \in \bar{J}}q_j} \quad \mbox{ for } j \in \bar{J}
\end{equation}
so that we only consider variables that have not been picked and the corresponding probabilities sum to 1. It is easy to see that $q_j$'s after normalization by \eqref{def_prob_renormalize} are strictly greater than $q_j$'s before normalization. Note also that $q_j$'s in \eqref{def_prob_renormalize} also satisfy \autoref{lemma_min_max_prob_ratio}, since in \eqref{def_prob_renormalize} we are multiplying them by a constant.

Now we are ready to show that \textit{Core-Random} with $\theta =1$ finds a global optimal solution with probability 1. We first precisely review how \textit{Core-Random} proceeds and define a detailed notation for the analysis. In iteration $t$, the following steps are performed.

\begin{enumerate}[noitemsep]
\item We solve \eqref{formulation_core_MAE} with $C$ in \textit{Core-Random} and obtain $S^*$. Note that the core set is from the previous iteration. Hence, we denote the core set as $C_{t-1}$.
\item In Step \ref{algo_update_coreset_random_line1} of \textit{Update-Core-Set-Random}, we check the neighborhood of $S^*$ obtained from \eqref{formulation_core_MAE} and update $S^*$ if applicable.
\item After Step \ref{algo_update_coreset_random_line1} of \textit{Update-Core-Set-Random}, we obtain $q_j$'s from \eqref{def_selection_prob}. Let $q_j^{(t)}$ be the initial probability, defined in \eqref{def_selection_prob}, used to construct the core set in iteration $t$.
\item In Step \ref{algo_update_coreset_random_line2} of \textit{Update-Core-Set-Random}, we construct core set $C_t$ based on $q_j^{(t)}$'s. Note that $C_t$ is used in iteration $t+1$ to solve \eqref{formulation_core_MAE}.
\end{enumerate}
Let $S^{opt}$ be an optimal subset. If $S^{opt} \subset C_t$ for a core set $C_t$, then we can find a global optimal solution by solving \eqref{formulation_core_MAE} in iteration $t+1$. We first derive a lower bound of the probability for the event $S_{opt} \subset C_t$ given any previous iterations.

\begin{lemma}
\label{lemma_prob_opt_set}
Let $\mathcal{H}_{t-1}$ be the set that includes any collection of the events that have happened prior to iteration $t$. Then, we have $$P[ S^{opt} \subset C_t |\mathcal{H}_{t-1} ]  \geq  \Big(\frac{1}{1+2.72(m-1)} \Big)^{\Theta}.$$
\end{lemma}
The proof is available in Section 2 of the online supplement. Let $mae_a^{opt}$ be the optimal objective function value of \eqref{formulation_MAE_more_m} over the entire $J$ and $mae_a(t)$ be the objective function value of the current best solution in iteration $t$ of \textit{Core-Random}, i.e., the objective value with respect to $S^*$. Let $A_t$ be the event $\{ S^{opt} \not\subset C_t \}$ in iteration $t$. For notational convenience, let 
\begin{center}
$\varphi = \Big(\frac{1}{1+2.72(m-1)} \Big)^{\Theta}$
\end{center}
be the lower bound for $P[ S^{opt} \subset C_t | \mathcal{H}_{t-1}]$ from \autoref{lemma_prob_opt_set}. Based on \autoref{lemma_prob_opt_set}, we present the following lemmas with the proofs given in Section 2 of the online supplement.

\begin{lemma}
\label{lemma_bound_no_opt_event}
We have $P \Big[\bigcap_{k=1}^t A_k \Big] \leq (1-\varphi)^t$ for any iteration $t$.
\end{lemma}

\begin{lemma}
\label{lemma_lowerbound_finding_opt}
We have $P\Big[ mae_a(t) = mae_a^{opt}\Big] \geq 1- (1-\varphi)^t$ for any iteration $t$.
\end{lemma}

Finally, we show that \textit{Core-Random} finds a global optimal solution with probability 1 as iterations continue infinitely.

\begin{proposition}
We have $\lim_{t \rightarrow \infty} P \Big[  mae_a(t) = mae_a^{opt} \Big] = 1$.
\end{proposition}
\begin{proof}
Since $0 < \varphi < 1$ by the definition of $\varphi$, we have $\lim_{t \rightarrow \infty} (1-\varphi)^t = 0$. Using this result, we derive
\begin{center}
$\lim_{t \rightarrow \infty} P \Big[  mae_a(t) = mae_a^{opt} \Big] \geq \lim_{t \rightarrow \infty} 1- (1-\varphi)^t = 1$.
\end{center}
Hence, we obtain $\lim_{t \rightarrow \infty} P \Big[  mae_a(t) = mae_a^{opt} \Big] = 1$.
\end{proof}




\section{Computational Experiment}
\label{REG_section_computational_experiment}

In this section, we present computational experiments for all proposed models and algorithms in \autoref{REG_section_formulation_subset_selection} and \autoref{section_many_variables}. They are compared to benchmark algorithms and to each other. To test the performance, we use randomly generated instances and a personal computer with 8 GB RAM and Intel Core i7 (2.40 GHz dual core) was used for the experiments in Section \ref{experiment_thin_mrmr} and a server with Xeon 2.8 GHz CPU and 15GB RAM is used for all other experiments. All models and algorithms are implemented in C\# and CPLEX.

\subsection{Experimental Design}

We obtained many publicly available instances for the subset selection problem. The majority of them were very easy to solve by both our models and stepwise heuristics. One of the purposes of this study is to establish the solution quality of the stepwise heuristic versus the optimal solutions. For these reasons, we generated synthetic instances. Furthermore, we want a large variety of instances with regard to the size and by randomly generating instances, we were also able to achieve this.

For the thin case ($m < n$),  we generate 26 sets of instances with $\{ (m,n) | m \in \{20,30,40,50\}, n \in \{30,40,\cdots,90,100\}, m + 10 \leq n\}$, where each set contains 10 instances. Hence, we generate a total of 260 instances. For the fat case ($m > n$), we generate 16 sets of instances with $\{ (m,n) | m \in \{100,150,200,250\}, n \in \{30,40,50,60\}\}$, in which each set contains 10 instances. Hence, we generate a total of 160 instances. For the detailed procedure used to generate the instances, see Section 8 of the online supplement.

To evaluate the performance of the proposed models and algorithms, we compare the improvement against benchmark packages and algorithms. For the thin case with $MAE$ objective and the fat case with both $MAE$ and $MSE$ objectives, we implemented a stepwise algorithm in C\#, due to the absence of a statistical package that supports such cases. The algorithm is presented in Section 3 of the online supplement. For the thin case with the $MSE$ objective, we use the stepwise regression implementation of R statistics package Leaps by \citet{Lumley:2009}, which supports the adjusted $r^2$ objective. The leaps package also provides leaps-and-bound, an exact algorithm proposed by \citet{Furnival-Wilson:74}. However, in Section 4 of the online supplement, we show that its complexity is much worse than that of our algorithms. For the remaining portion of the paper, we refer to all of the benchmark algorithms and packages as \textit{Step}. For all proposed models and algorithms, solutions obtained by \textit{Step} are used as initial solutions. As we discussed in the introduction, enumerating all possible subsets is not a computationally tractable approach and it is excluded in the comparison.

For comparison purposes, we use the following measures.
\begin{enumerate}[noitemsep]
\item[] $\mbox{GAP}_{IP}$: the optimality gap obtained by CPLEX within allowed time. 
\item[] $\mbox{GAP}_{sol}$: relative gap between a proposed model and heuristic defined as $$\frac{obj\mbox{ of } Step - obj \mbox{ of proposed model}}{obj\mbox{ of } Step}.$$ 
\end{enumerate}

Solving the problems optimally for larger instances takes a long time as implied in Section 4 of the online supplement. Hence, we set up time limits for CPLEX. We execute CPLEX with two settings for the time limit: one hour and one minute. The computation time of the big $M$ is less than 90 seconds for all instances considered in the experiment, and we do not include this time within the one hour and one minute time limits.

Finally, we summarize the algorithms used for the experiment in \autoref{tab:summary_algorithm}. Recall that we only presented the result for big M with the $MAE$ and $MAE_a$ objectives. For the $MSE$ and $MSE_a$ objectives, we need a trivial modification. In all algorithms and models, to obtain big M for $v_j$, we use \eqref{formula_bigM_for_v} and \eqref{formula_bigM_v_fat_case} for the thin and fat cases, respectively. However, we have several options to obtain the big M value for $x_j$: \eqref{bigM_definition}, \eqref{formula_bigM_heur}, and procedures in \autoref{appendix_big_M_alternatives}. Among these, for the thin case and each iteration of \textit{CoreHeur} and \textit{CoreRnd} for the fat case, we use \eqref{bigM_definition} for big M for $x_j$, because in each iteration we deal with the thin case. For the fat case MIP models, we use \eqref{formula_bigM_heur} for big M for $x_j$ because other procedures give extremely large values of $M$. These choices were made based on computational experiments in Section 5 of the online supplement. The result in the online supplement implies that valid big M values guarantee optimality while they do not significantly increase the execution times. Even if CPLEX terminates due to the time limit, the solution qualities are similar regardless of the big M values as long as the big M values are valid.

\begin{table}[htbp]
  \centering
  \begin{scriptsize}
    \begin{tabular}{|c|c|c|l|}
    \hline
    Case  & Obj   & Notation & Reference\\ \hline
    Thin  & $MAE$   & MIP   &   \eqref{formulation_subset_mae} with big M based on \eqref{formula_bigM_for_v} and \eqref{bigM_definition} \\ \hline
    Thin  & $MSE$   & MIP   &   \eqref{formulation_subset_mse}   with big M based on \eqref{formula_bigM_for_v} and \eqref{bigM_definition}\\ \hline
    Thin  & mRMR   & MIP   &   \eqref{formulation_mae_mrmr}, \eqref{formulation_mae_mrmr} does not have big M\\ \hline
    Fat   & $MAE_a$ & MIP   &   \eqref{formulation_MAE_more_m}  with big M based on \eqref{formula_bigM_v_fat_case} and \eqref{formula_bigM_heur}  \\
          &       & CoreHeur &  \autoref{algo_excessive_m} with \autoref{algo_update_coreset} and big M based on \eqref{formula_bigM_for_v} and \eqref{bigM_definition} with $J := C$\\
          &       & CoreRnd &   \autoref{algo_excessive_m} with  \autoref{algo_update_coreset_random} and big M based on \eqref{formula_bigM_for_v} and \eqref{bigM_definition} with $J := C$\\ \hline
    Fat   & $MSE_a$ & MIP   &    \eqref{formulation_MSE_more_m} with big M based on \eqref{formula_bigM_v_fat_case} and \eqref{formula_bigM_heur}   \\\
          &       & CoreHeur &  \autoref{algo_excessive_m} with \autoref{algo_update_coreset} and big M based on \eqref{formula_bigM_for_v} and \eqref{bigM_definition} with $J := C$\\
          &       & CoreRnd &    \autoref{algo_excessive_m} with \autoref{algo_update_coreset_random} and big M based on \eqref{formula_bigM_for_v} and \eqref{bigM_definition} with $J := C$\\ \hline
    \end{tabular}%
  \end{scriptsize}   
\caption{Summary of the algorithms} \label{tab:summary_algorithm}
\end{table}%

We also note here that big M-based formulations we propose outperform logical constraint-based formulations that are available in CPLEX and most commercial optimization solvers. In Section 6 of the online supplement, we compare the two approaches and observe that the proposed formulations terminate faster with an optimal solution  or terminate with a better solution (smaller optimality gap and smaller objective function value) when one minute time limit is employed.

\subsection{Study of Thin Case \texorpdfstring{($m < n$)}{(m < n)} for MAE and MSE Objectives}

In \autoref{fig:result_1hr}, we present the averages of $\mbox{GAP}_{IP}$ and $\mbox{GAP}_{sol}$ across the 26 instance sets. Each rectangle and circle corresponds to the average $\mbox{GAP}_{IP}$ and $\mbox{GAP}_{sol}$ of 10 instances for the corresponding instance set. In both plots on the left, x and y axes represent the instance sets and the gaps in percentage. For both $MSE$ and $MAE$, $\mbox{GAP}_{IP}$ is near zero for most of the instances with $m \leq 40$. Hence, we get an optimal solution within one hour. For larger instances, $\mbox{GAP}_{IP}$ is positive for both $MSE$ and $MAE$ and is larger for $MSE$. For $\mbox{GAP}_{sol}$, we observe common phenomena for both objectives. First, $\mbox{GAP}_{sol}$ tends to decrease as $n$ increases for each fixed $m$. Second, there are bumps for $\mbox{GAP}_{sol}$ at $(m,n) \in \{(20,30),(30,40),(40,50),(50,60)\}$. \autoref{fig:result_1hr} also implies that the performance of heuristics deteriorates when we have relatively fewer observations given fixed $m$, because $\mbox{GAP}_{sol}$ is an underestimation of the gap between an optimal solution and heuristic solution. We also plot the average execution time of \eqref{formulation_subset_mae} and \eqref{formulation_subset_mse}. Observe that the average time of \eqref{formulation_subset_mae} for large instances is still 500 seconds, while $\mbox{GAP}_{IP}$ is positive for the same instance sets. This implies that most of the instances are solved optimally but we terminate with a relatively large $\mbox{GAP}_{IP}$ for a few instances after one hour.

\begin{figure}[ht]
\centering
		\subfigure[$\mbox{GAP}_{IP}$ and $\mbox{GAP}_{sol}$ for $MSE$]{%
           \includegraphics[scale=0.4]{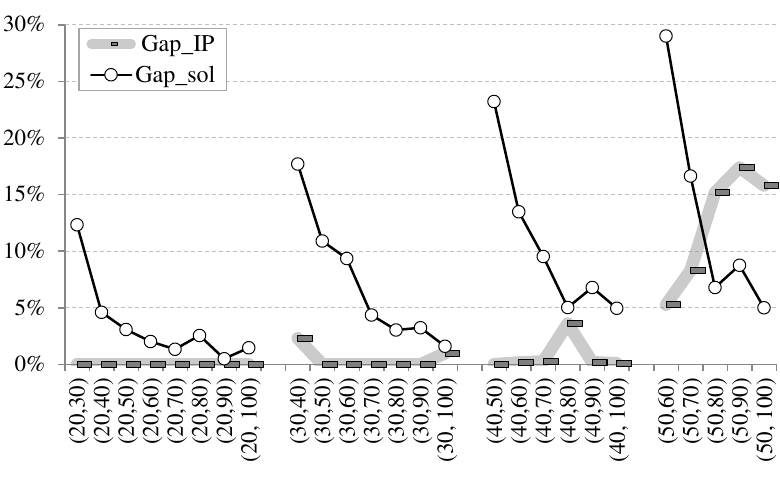} \label{fig:result_mse}
        }
        \subfigure[$\mbox{GAP}_{IP}$ and $\mbox{GAP}_{sol}$ for $MAE$]{%
           \includegraphics[scale=0.4]{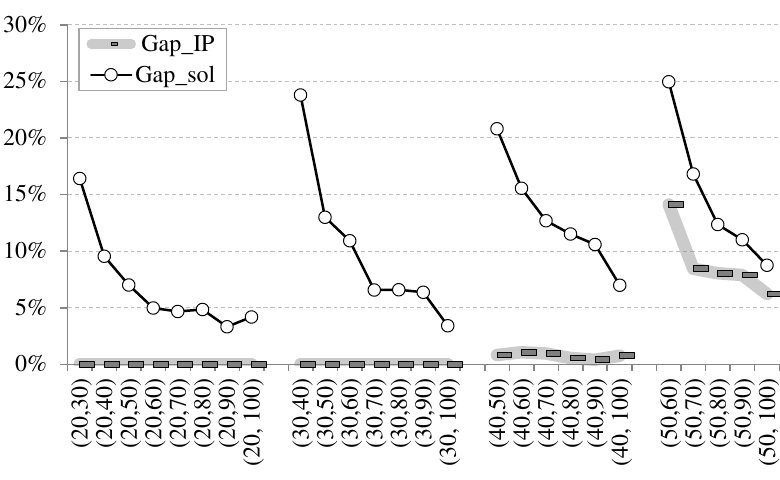} \label{fig:result_mae}
        }
        \subfigure[MIP Execution time]{%
           \includegraphics[scale=0.4]{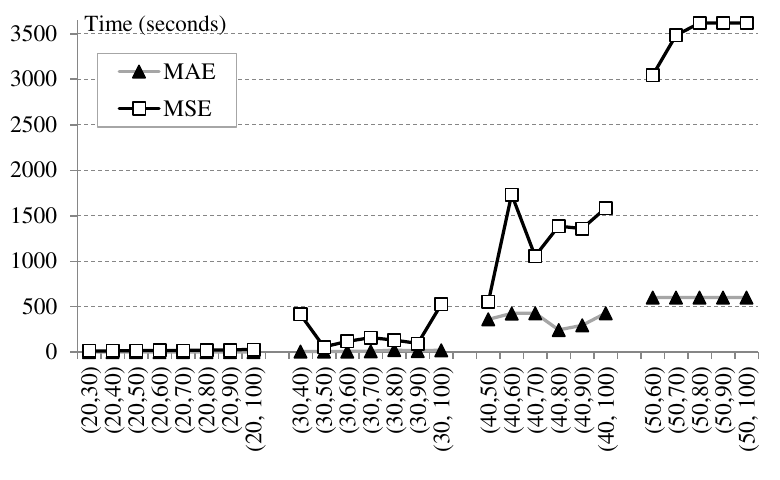} \label{fig:result_thin_time}
        }
\caption{Average $\mbox{GAP}_{IP}$, $\mbox{GAP}_{sol}$, and execution time with the one hour  time limit}\label{fig:result_1hr}
\end{figure}

During the experiment, we observed that the improvement of the objective function value occurs in the early stage of the branch-and-bound algorithm, and CPLEX tries to improve the lower bound for the remaining time. In \autoref{fig:fg_why_1min}, we present the primal and lower bounds for one instance over time. The circles and empty circles are the primal and lower bounds over time, respectively, and the plain and dotted lines represent the best primal and lower bounds obtained after one hour. Observe that there is no objective function value improvement after 90 and 25 seconds for $MSE$ and $MAE$, respectively. In other words, we can obtain the same regression models obtained with one hour execution by terminating CPLEX after 90 seconds. From this observation, we conclude that good solutions are obtained in the early stages of the branch-and-bound algorithm but improving the lower bound takes longer time. This observation gives the justification to run CPLEX for a short time if we do not need to retain optimality.

\begin{figure}[ht]
\centering

\subfigure[$MSE$]{%
           \includegraphics[scale=0.45]{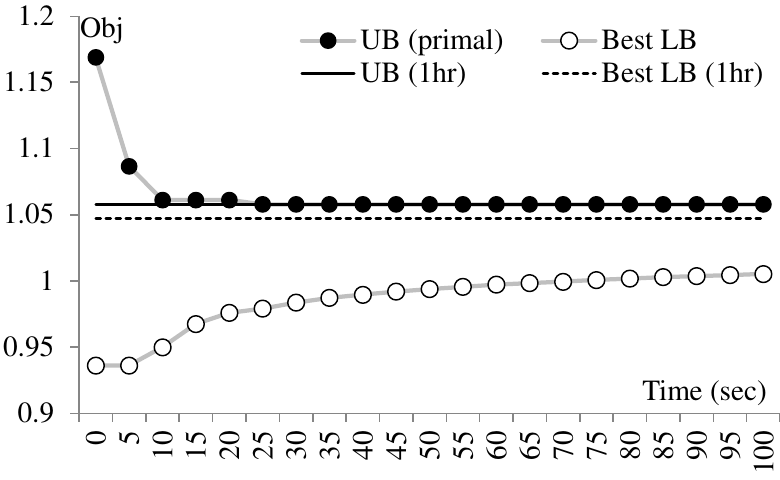} \label{fig:why_1min_mse}
        }\qquad
        \subfigure[$MAE$]{%
           \includegraphics[scale=0.45]{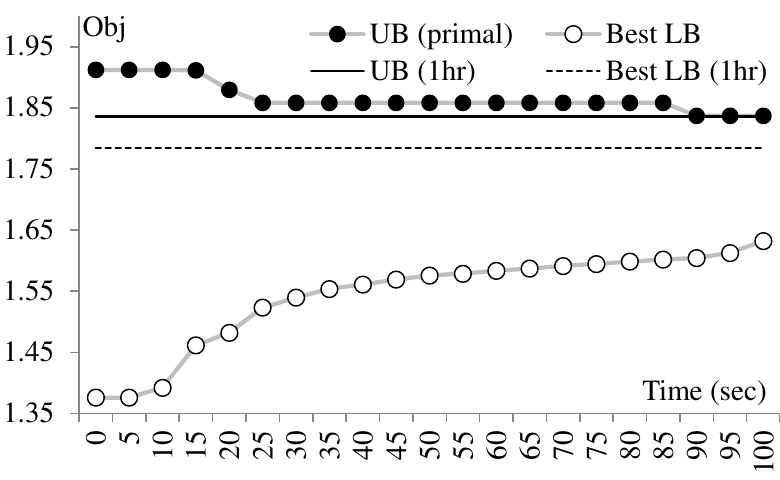} \label{fig:why_1min_mae}
        }     
\caption{Convergence of primal and dual bounds for an instance with $m=50$ and $n=100$}\label{fig:fg_why_1min}
\end{figure}

For this reason, we execute CPLEX with the one minute time limit. In the experiment of \citet{Bertsimas-etal:15}, time limit of 500 seconds for MIP is considered as they solve different formulation with larger data. In \autoref{fig:result_1min}, we present the averages of $\mbox{GAP}_{IP}$ and $\mbox{GAP}_{sol}$ over 26 instance sets, when CPLEX terminates after one minute. We observe a similar shape for $\mbox{GAP}_{sol}$ except the gaps are slightly smaller. On the other hand, $\mbox{GAP}_{IP}$ is positive for more instances compared to the previous result with the one hour time limit. To compare the solution qualities precisely, in \autoref{fig:result_lost}, we plot the improvement of the primal and lower bounds obtained by executing the extra 59 minutes, where the data points represent $lost(\mbox{GAP}_{sol}) =$ $\big( \mbox{GAP}_{sol}$ with one hour - $\mbox{GAP}_{sol}$ with one minute$\big)$ and $lost(\mbox{GAP}_{IP}) =$ $\big( \mbox{GAP}_{IP}$ with one minute - $\mbox{GAP}_{IP}$ with one hour$\big)$. Observe that the difference of $\mbox{GAP}_{sol}$ is less than 5\% for all cases, whereas there exists significant improvement of the lower bounds for $m \geq 30$. Therefore, within one minute (excluding the big M time), we can improve the stepwise heuristic solution up to 25\% by solving the proposed MIP models.

\begin{figure}[ht]
\centering
		\subfigure[$MSE$]{%
           \includegraphics[scale=0.5]{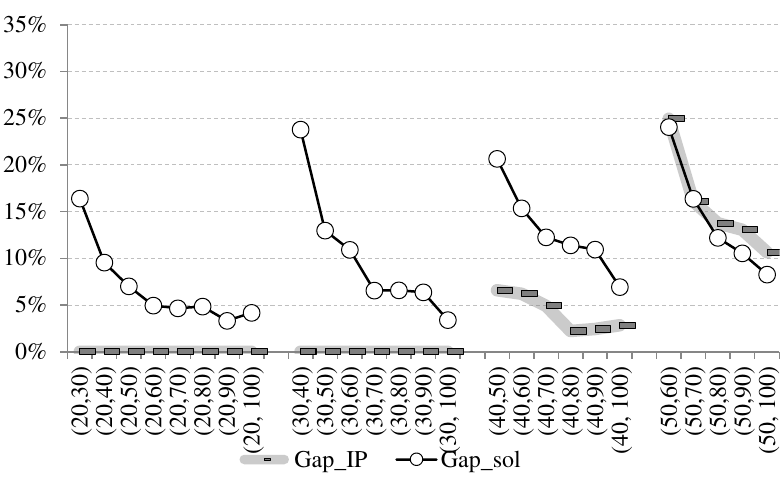} \label{fig:result_mse_1min}
        }\quad
        \subfigure[$MAE$]{%
           \includegraphics[scale=0.5]{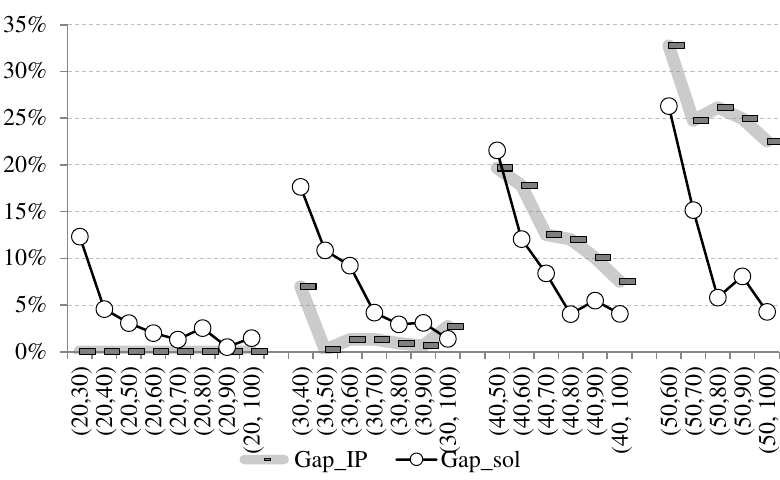} \label{fig:result_mae_1min}
        }  
        
\caption{Average $\mbox{GAP}_{IP}$ and $\mbox{GAP}_{sol}$ with the one minute time limit}\label{fig:result_1min}

\subfigure[$MSE$]{%
           \includegraphics[scale=0.5]{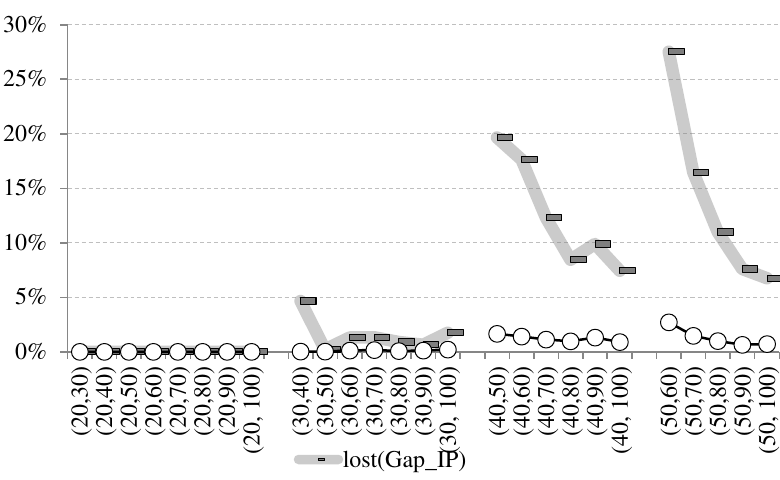} \label{fig:result_mse_lost}
        }\quad
        \subfigure[$MAE$]{%
           \includegraphics[scale=0.5]{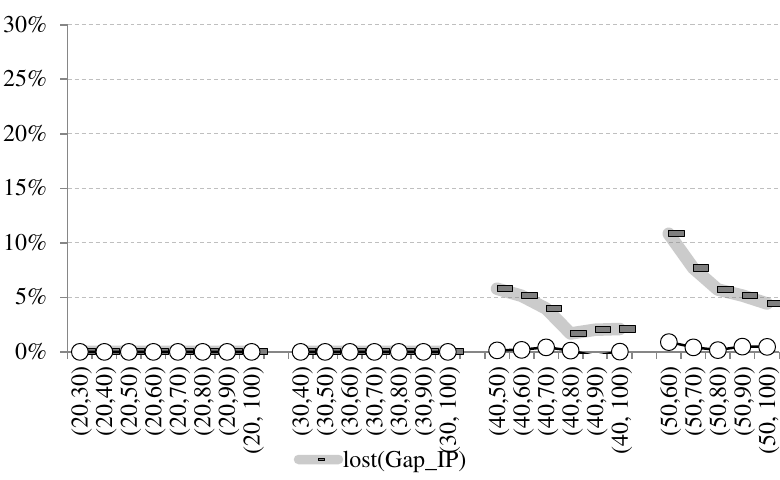} \label{fig:result_mae_lost}
        }      
\caption{Average improvement of $\mbox{GAP}_{IP}$ and $\mbox{GAP}_{sol}$ by the extra 59 minutes}\label{fig:result_lost}
\end{figure}

\subsection{Study of Thin Case \texorpdfstring{($m < n$)}{(m < n)} for Minimal-Redundancy-Maximal-Relevance}
\label{experiment_thin_mrmr}

In this section, four MIP models are compared: MIP$_{\mbox{\scriptsize{mrmr}}}$ (MIP model \eqref{formualtion_IG_initial} maximizing mRMR), MIP$_{\mbox{\scriptsize{mae}}}$ (MIP model \eqref{formulation_subset_mae}), MIP$_{\mbox{\scriptsize{sae}}}$ (MIP model \eqref{formulation_subset_mae} with fixed $p$ minimizing SAE), and MIP$_{\mbox{\scriptsize{mix}}}$ (MIP model \eqref{formulation_mae_mrmr} minimizing SAE subject to the mRMR constraint).

In the first experiment, MIP$_{\mbox{\scriptsize{mix}}}$ is compared with MIP$_{\mbox{\scriptsize{mrmr}}}$ and MIP$_{\mbox{\scriptsize{sae}}}$ for fixed $p$ values. In the second experiment, MIP$_{\mbox{\scriptsize{mix}}}$ is compared with MIP$_{\mbox{\scriptsize{mae}}}$. Let $S_{\mbox{\scriptsize{mrmr}}}$, $S_{\mbox{\scriptsize{sae}}}$, and $S_{\mbox{\scriptsize{mix}}}$ be the selected subsets of the corresponding MIP models, and let mRMR$_{\mbox{\scriptsize{mrmr}}}$ and mRMR$_{\mbox{\scriptsize{mixed}}}$ be the mRMR values for $S_{\mbox{\scriptsize{mrmr}}}$ and $S_{\mbox{\scriptsize{mix}}}$, respectively. Let SAE$_{\mbox{\scriptsize{sae}}}$ and SAE$_{\mbox{\scriptsize{mixed}}}$ be the SAE values for $S_{\mbox{\scriptsize{sae}}}$ and $S_{\mbox{\scriptsize{mix}}}$, respectively. To compare the selected subset and solution quality of MIP$_{\mbox{\scriptsize{mix}}}$ against the other three models, the following criteria are used. For each $model \in \{mrmr, mae, sae\}$, set difference between $S_{\mbox{\scriptsize{model}}}$ and $S_{\mbox{\scriptsize{mixed}}}$, SD$_{\mbox{\scriptsize{model}}} = \frac{|(S_{\mbox{\scriptsize{model}}} \setminus S_{\mbox{\scriptsize{mix}}})| + |(S_{\mbox{\scriptsize{mix}}} \setminus S_{\mbox{\scriptsize{model}}})|}{2}$, is defined. For all four models, the relative mRMR gap from MIP$_{\mbox{\scriptsize{mrmr}}}$ (GAP$_{\mbox{\scriptsize{mrmr}}} (\%) = \frac{\mbox{mRMR}_{\mbox{\scriptsize{mrmr}}}-\mbox{mRMR}_{\mbox{\scriptsize{model}}}}{\mbox{mRMR}_{\mbox{\scriptsize{mrmr}}}} \times 100$) and relative SAE gap from the optimal SAE (GAP$_{\mbox{\scriptsize{sae}}} (\%) = \frac{\mbox{SAE}_{\mbox{\scriptsize{model}}}-\mbox{SAE}_{\mbox{\scriptsize{sae}}}}{\mbox{SAE}_{\mbox{\scriptsize{sae}}}} \times 100$) are defined. Note that SD$_{\mbox{\scriptsize{mrmr}}}$ and SD$_{\mbox{\scriptsize{sae}}}$ measure how the selected subset by MIP$_{\mbox{\scriptsize{mix}}}$ is different from the subsets obtained by MIP$_{\mbox{\scriptsize{mrmr}}}$ and MIP$_{\mbox{\scriptsize{sae}}}$, respectively. To measure the solution quality in terms of mRMR and SAE, GAP$_{\mbox{\scriptsize{mrmr}}}$ and GAP$_{\mbox{\scriptsize{sae}}}$ calculate the relative gaps of MIP$_{\mbox{\scriptsize{mix}}}$ from the best mRMR (by MIP$_{\mbox{\scriptsize{mrmr}}}$) and best SAE (by MIP$_{\mbox{\scriptsize{sae}}}$), respectively. 

To test the performances of the models with various parameters and sizes, we conduct experiments using the thin case synthetic data from Section 4.2 and report the result in Section 9 of the online supplement. The result of these experiments confirms that MIP$_{\mbox{\scriptsize{mix}}}$ effectively balances the mRMR and SAE objects. The obtained subset by MIP$_{\mbox{\scriptsize{mix}}}$ is distinguished from the subsets of MIP$_{\mbox{\scriptsize{mrmr}}}$ and MIP$_{\mbox{\scriptsize{sae}}}$. Check the online supplement for the detailed results. For the experiments in this section, the MIP models are tested using select real datasets from the UCI Machine Learning Repository \cite{Lichman:2013} and \citet{johnson1996fitting}. Four regression datasets (Bodyfat, Autompg, Housing, and Servo) are selected among the datasets with more than 100 observations and that are created for linear regression analysis. The original data are processed by deleting rows with missing values and by creating dummy variables for categorical variables. All final variables are standardized.

In the first experiment, for each dataset, parameters $p \in \{3,4,5,6\}$ and $\lambda \in \{0.1,0.2,0.3,0.4,0.5\}$ are used. In Figure \ref{fg_exp_mrmr_uci}, a heatmap is presented for the four performance measures SD$_{\mbox{\scriptsize{mrmr}}}$, SD$_{\mbox{\scriptsize{sae}}}$, GAP$_{\mbox{\scriptsize{mrmr}}}$, and GAP$_{\mbox{\scriptsize{sae}}}$. The execution times are not reported because all models are solved optimally within a second. The rows are defined for datasets and $p$, and the columns are defined for $\lambda$ values. The heatmap shows the same trend with the previous experiments. Increasing $\lambda$ and $p$ values increases SD$_{\mbox{\scriptsize{mrmr}}}$ and GAP$_{\mbox{\scriptsize{mrmr}}}$ while decreases SD$_{\mbox{\scriptsize{sae}}}$ and GAP$_{\mbox{\scriptsize{sae}}}$. For several cases (Housing data with $p = 3,4,5$), SD$_{\mbox{\scriptsize{mrmr}}} = $ SD$_{\mbox{\scriptsize{sae}}} = 0$ because the selected subset is optimal for both criteria mRMR and SAE. For several cases ($\lambda = 0.4, 0.5$ for Augompg, Housing, Servo), SD$_{\mbox{\scriptsize{sae}}} = 0$ because $S_{\mbox{\scriptsize{sae}}}$ has the mRMR value within 40\% from the optimal mRMR value, which also implies Constraint \eqref{formulation_mae_mrmr_e} does not cut any part of the feasible region. In order to determine the best balance between the two criteria, a user can determine an allowable maximum for any of the gaps GAP$_{\mbox{\scriptsize{mrmr}}}$ and GAP$_{\mbox{\scriptsize{sae}}}$ and select the best in the scope. Otherwise, a pareto frontier and scatter plot can be useful in selecting a good solution. 

\begin{figure}[ht]
\centering
    \includegraphics[scale=0.85]{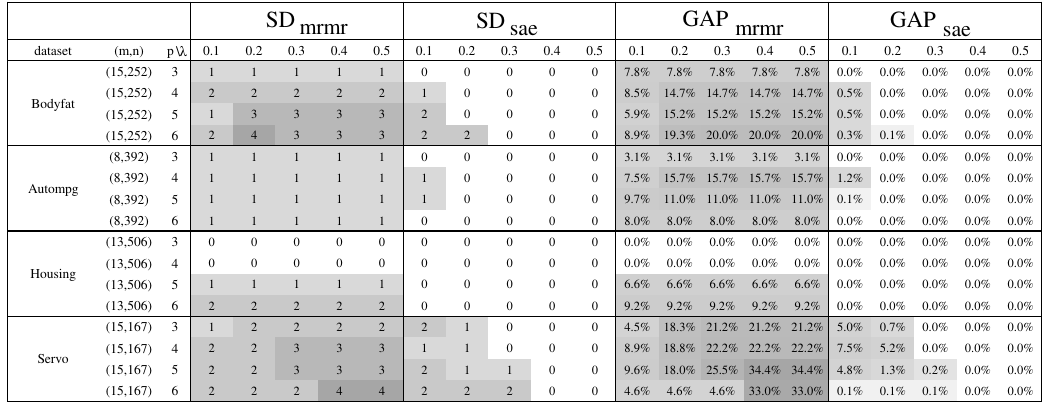} 
    \vspace{-0.3cm}
    \caption{Performances of MIP$_{\mbox{\scriptsize{mix}}}$ compared to MIP$_{\mbox{\scriptsize{mrmr}}}$ and MIP$_{\mbox{\scriptsize{sae}}}$}
\label{fg_exp_mrmr_uci}
\end{figure}

In the second experiment, MIP$_{\mbox{\scriptsize{mix}}}$ is compared with our MIP model \eqref{formulation_subset_mae}, which we denote as MIP$_{\mbox{\scriptsize{mae}}}$. While MIP$_{\mbox{\scriptsize{sae}}}$ assumes fixed $p$, our MIP model (7) from Section 2.1 can be used to find the optimal $p$ value, referred to as $p^*$. Hence, we solved (7) to obtain $p^*$ and the optimal $MAE$. Then, we compare the solution quality of MIP$_{\mbox{\scriptsize{mix}}}$ by fixing $p$ to $p^*$ and by checking various $\lambda$ values. In Table \ref{table_exp_mae_vs_mixs}, the fourth column represents the relative gap of the mRMR objective between (7) and optimal mRMR, the fifth column represents the relative gap of the $MAE$ objective between (7) and MIP$_{\mbox{\scriptsize{mix}}}$, the sixth column represents the relative gap of the mRMR objective between MIP$_{\mbox{\scriptsize{mix}}}$ and optimal mRMR, and the last column represents the set difference between (7) and MIP$_{\mbox{\scriptsize{mix}}}$.

\begin{table}[htbp]
  \centering
\begin{scriptsize}

    \begin{tabular}{|rrr|rr|rrr|}
    \hline
          &     &   & \multicolumn{2}{c|}{MIP$_{\mbox{\scriptsize{mae}}}$}              & \multicolumn{3}{c|}{MIP$_{\mbox{\scriptsize{mix}}}$}         \\ \hline
    Dataset & ($m,n$) & $\lambda$ & p*  & GAP$_{\mbox{\scriptsize{mrmr}}}$ & GAP$_{\mbox{\scriptsize{mae}}}$ & GAP$_{\mbox{\scriptsize{mrmr}}}$ & SD$_{\mbox{\scriptsize{mae}}}$ \\ \hline
    Bodyfat & (15,252) & 0.05 & 4 &  14.7\% & 0.6\% & 4.7\% & 2 \\
          &      & 0.1 &             &       & 0.5\% & 8.5\% & 1 \\
          &      & 0.15 &             &       & 0.0\% & 14.7\% & 0 \\
          &      & 0.2 &              &       & 0.0\% & 14.7\% & 0 \\ \hline
   Autompg & (8,392) & 0.05 & 4  & 15.7\% & 1.2\% & 1.0\% & 1 \\
          &       & 0.1&              &       & 1.2\% & 7.5\% & 1 \\
          &       & 0.15&              &       & 1.2\% & 7.5\% & 1 \\
          &       & 0.2&             &       & 0.0\% & 15.7\% & 0 \\ \hline
    Housing & (13,506) & 0.05 & 11  & 6.5\% & 1.3\% & 4.7\% & 2 \\
          &       & 0.1&              &       & 0.0\% & 6.5\% & 0 \\
          &       & 0.15&              &       & 0.0\% & 6.5\% & 0 \\
          &       & 0.2&              &       & 0.0\% & 6.5\% & 0 \\ \hline
    Servo & (15,167) & 0.05 & 9  & 8.1\% & 1.1\% & 4.8\% & 1 \\
          &       & 0.1&              &       & 0.0\% & 8.1\% & 0 \\
          &       & 0.15&              &       & 0.0\% & 8.1\% & 0 \\
          &       & 0.2&              &       & 0.0\% & 8.1\% & 0 \\ \hline
    \end{tabular}%
\end{scriptsize}
  \caption{Comparison with MAE model}
  \label{table_exp_mae_vs_mixs}%
\end{table}%

The GAP$_{\mbox{\scriptsize{mrmr}}}$ values of MIP$_{\mbox{\scriptsize{mae}}}$ show that the optimal $MAE$ subset is quite different from the optimal mRMR subset and the mRMR values are different up to 14.7\%. By MIP$_{\mbox{\scriptsize{mix}}}$, we can improve the mRMR value significantly without decreasing $MAE$ too much. For all four datasets, with $\lambda = 0.05$, GAP$_{\mbox{\scriptsize{mrmr}}}$ values of MIP$_{\mbox{\scriptsize{mix}}}$ are significantly lower than those of MIP$_{\mbox{\scriptsize{mae}}}$, while GAP$_{\mbox{\scriptsize{mae}}}$ values of MIP$_{\mbox{\scriptsize{mix}}}$ are approximately 1\% from the optimal $MAE$ value. In detail, for Autompg data, MIP$_{\mbox{\scriptsize{mix}}}$ keeps both of GAP$_{\mbox{\scriptsize{mae}}}$ and GAP$_{\mbox{\scriptsize{mrmr}}}$ approximately at 1\%.

\subsection{Study of Fat Case \texorpdfstring{($m > n$)}{(m > n)}}

In this section, we present two experiments for the fat case datasets. In the first experiment, the solution qualities of the MIP models, \eqref{formulation_MAE_more_m} and \eqref{formulation_MSE_more_m}, and the core set algorithms, \textit{Core-Heuristic} and \textit{Core-Random}, are compared using the synthetic datasets. In the second experiment, the core set algorithms are compared against the stepwise algorithm and a state-of-the-art benchmark algorithm using real-world instances from the UCI Machine Learning Repository. 

Recall that the core set algorithms require core set cardinality parameter $\theta$. Hence, we first decide the best $\theta$ value for each core set algorithm, then we compare \textit{Core-Heuristic}, \textit{Core-Random}, and the MIP models in Section 7 of the online supplement. We conclude the following universal rule for the selection of $\theta$.
\begin{enumerate}[noitemsep]
\item For \textit{Core-Heuristic}, we use $\theta = 1$ for instance sets satisfying $\{ \frac{n}{m} \geq 0.4, n \leq 40 \}$ or $\{ \frac{n}{m} \geq 0.5, n > 40 \}$. For all other instances, we use $\theta=0.8$.
\item For \textit{Core-Random}, with a ten minute time limit, $\theta = 1.0$ is best for all sizes. 
\item For \textit{Core-Random}, with a one hour time limit, $\theta = 0.8$ is best for large instances. Hence, with the one hour time limit, we use $\theta = 0.8$ if $mn \geq 9000$ and $\theta = 1.0$ otherwise.
\end{enumerate}

We compare $GAP_{sol}$ of the MIP models, and \textit{Core-Heuristic} and \textit{Core-Random} with the best $\theta$ determined by the rule above. In \autoref{fig:result_lost}, we observed that running the MIP solver beyond 1 minute does not improve the solution quality much. For this reason, to save computational power, we ran the MIP solver for 1 minute for the fat case. For \textit{Core-Random}, we set 10 minutes and 1 hour time limit to check the performance as we spend more time. 

For the first experiment, we present the average $GAP_{sol}$ for all algorithms and execution times for \textit{Core-Heuristic} in \autoref{fig:fat_case}. For the $MSE_a$ objective, MIP performs worst for all instances. For many instance sets, it does not improve the initial heuristic solution. \textit{Core-Random} performs slightly better than \textit{Core-Heuristic} for small instances with $n=30$, but they perform equally for remaining instances. For the $MAE_a$ objective, the performance of MIP drops substantially when $m$ increases. For most instances, \textit{Core-Random} performs the best in general. However, for larger instances with $n=60$, \textit{Core-Heuristic} performs the best.

\begin{figure}[ht]
\begin{center}
		\subfigure[$GAP_{sol}$ for $MSE_a$]{%
           \includegraphics[scale=0.4]{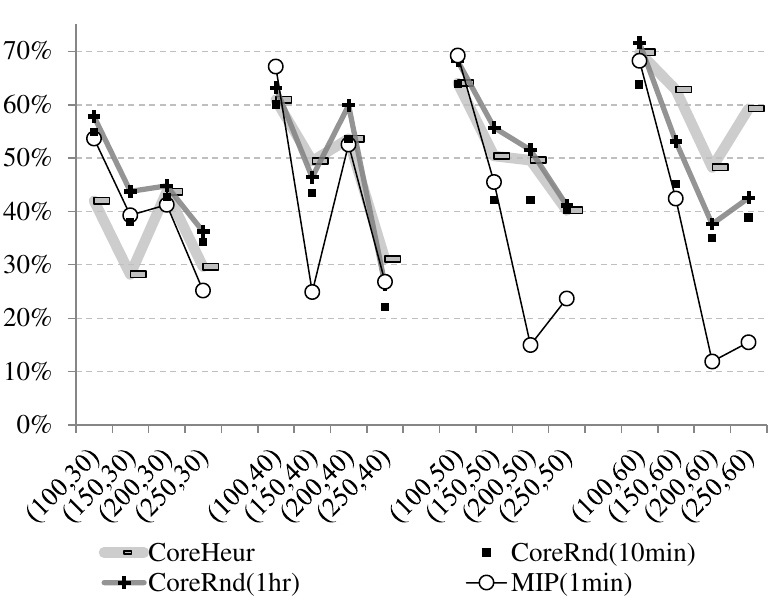} \label{fig:fat_case_mse_a}
        }
        \subfigure[$GAP_{sol}$ for $MAE_a$]{%
           \includegraphics[scale=0.4]{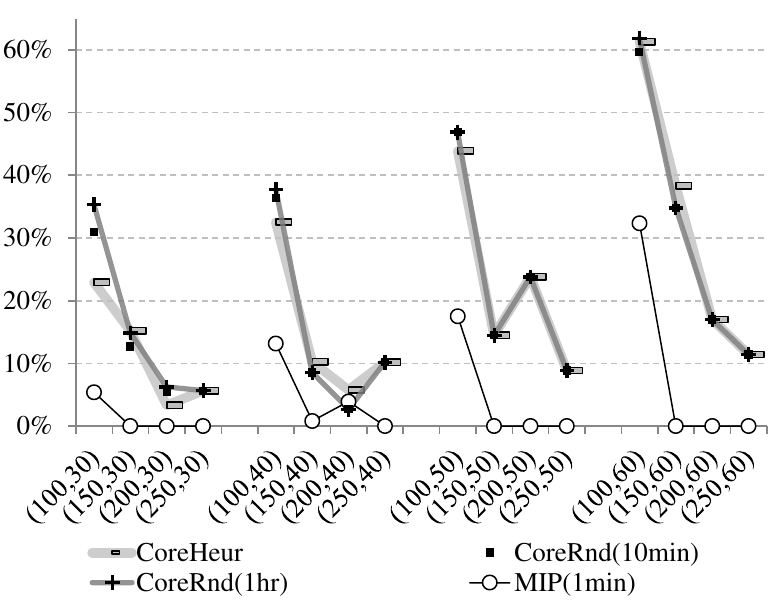} \label{fig:fat_case_mae_a}
        }
        \subfigure[\textit{Core-Heuristic} execution time]{%
           \includegraphics[scale=0.4]{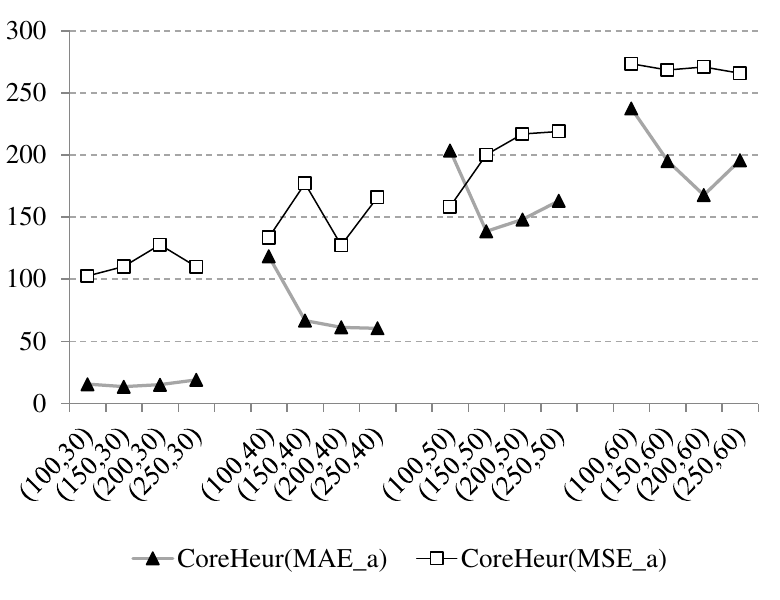} \label{fig:fat_coreheur_time}
        }  
\end{center}      
\caption{Comparison of performance of the algorithms}\label{fig:fat_case}
\end{figure}

For the second experiment, we compare the performance of the core set algorithms with two benchmark algorithms: a stepwise heuristic minimizing $MSE$ and the mathematical programming based algorithm of \citet{Bertsimas-etal:15}. We use the R package \textit{bestsubset} of \citet{Hastie-etal} which implements \citet{Bertsimas-etal:15}. We denote this algorithm as \textit{BKM}. For this experiment, we use two sets of the dataset of \citet{Rafiei:15} from the UCI Machine Learning Repository \cite{Lichman:2013} that have more than 100 features and that are created for linear regression analysis. The original dataset has 103 features and two possible response variables cost and sales. To create fat case datasets, we randomly select 50 observations and create 10 instances for each response variable. All explanatory variables are standardized.

We use a ten minute time limit for BKM to compare with the core set algorithms. Note that, within this time limit, BKM may not guarantee optimality and that BKM requires a fixed $p$. To search for the best $MSE$ and within the ten minute time limit, we enumerate BKM with the following search order for $p$: 1,3,5,7, $\cdots$,45,47,2,4,6,$\cdots$,46,48. For each $p$, 60 seconds is allowed and the algorithm stops at 600 seconds even if all $p$ values were not searched.

\begin{table}[htbp]
  \centering
  \scriptsize
    \begin{tabular}{|cc|rrrrrrrr|}
	\hline
    \multicolumn{2}{|c|}{} & \multicolumn{4}{c|}{Gap from the best} & \multicolumn{4}{c|}{Time (seconds)} \\
    \hline
    data  & (n,m) & \multicolumn{1}{c}{BKM} & \multicolumn{1}{c}{Step} & \multicolumn{1}{c}{CoreHeur} & \multicolumn{1}{c|}{CoreRnd} & \multicolumn{1}{c}{BKM} & \multicolumn{1}{c}{Step} & \multicolumn{1}{c}{CoreHeur} & \multicolumn{1}{c|}{CoreRnd} \\
          &  & \multicolumn{1}{c}{(10min)} &  &  & \multicolumn{1}{c|}{(10min)} & \multicolumn{1}{c}{(10min)} &  & & \multicolumn{1}{c|}{(10min)} \\
    \hline
    Cost1  & (50,103) & 33.5\% & 105.2\% & 12.9\% & \multicolumn{1}{r|}{\textbf{0.0\%}} & 600   & 10    & 77    & 648 \\
    Cost2  & (50,103) & 53.8\% & 6.9\% & \textbf{0.0\%} & \multicolumn{1}{r|}{\textbf{0.0\%}} & 600   & 12    & 98    & 646 \\
    Cost3  & (50,103) & 5.4\% & 23.2\% & 8.5\% & \multicolumn{1}{r|}{\textbf{0.0\%}} & 600   & 10    & 141   & 632 \\
    Cost4  & (50,103) & 5.4\% & 17.2\% & \textbf{0.0\%} & \multicolumn{1}{r|}{8.9\%} & 600   & 9     & 204   & 634 \\
    Cost5  & (50,103) & 23.9\% & 14.4\% & \textbf{0.0\%} & \multicolumn{1}{r|}{3.0\%} & 600   & 11    & 141   & 647 \\
    Cost6  & (50,103) & 16.5\% & 50.1\% & \textbf{0.0\%} & \multicolumn{1}{r|}{5.8\%} & 600   & 4     & 134   & 624 \\
    Cost7  & (50,103) & \textbf{0.0\%} & 30.8\% & 26.7\% & \multicolumn{1}{r|}{3.1\%} & 600   & 8     & 136   & 634 \\
    Cost8  & (50,103) & 97.2\% & 18.7\% & 13.0\% & \multicolumn{1}{r|}{\textbf{0.0\%}} & 600   & 11    & 52    & 581 \\
    Cost9  & (50,103) & 7.3\% & 11.2\% & 3.8\% & \multicolumn{1}{r|}{\textbf{0.0\%}} & 600   & 7     & 137   & 633 \\
    Cost10  & (50,103) & 19.6\% & 10.4\% & 4.8\% & \multicolumn{1}{r|}{0.0\%} & 602   & 11    & 142   & 650 \\
    \hline
    Sales1  & (50,103) & 10.7\% & 28.3\% & \textbf{0.0\%} & \multicolumn{1}{r|}{1.6\%} & 602   & 6     & 138   & 628 \\
    Sales2  & (50,103) & 73.8\% & 363.0\% & 2.0\% & \multicolumn{1}{r|}{\textbf{0.0\%}} & 601   & 8     & 148   & 668 \\
    Sales3  & (50,103) & \textbf{0.0\%} & 44.1\% & 44.0\% & \multicolumn{1}{r|}{44.0\%} & 600   & 6     & 138   & 638 \\
    Sales4  & (50,103) & 1.3\% & 32.4\% & 2.3\% & \multicolumn{1}{r|}{\textbf{0.0\%}} & 600   & 5     & 135   & 642 \\
    Sales5  & (50,103) & 7.8\% & 4.6\% & 4.6\% & \multicolumn{1}{r|}{\textbf{0.0\%}} & 600   & 12    & 141   & 652 \\
    Sales6  & (50,103) & \textbf{0.0\%} & 29.9\% & 29.9\% & \multicolumn{1}{r|}{4.2\%} & 600   & 8     & 139   & 645 \\
    Sales7  & (50,103) & \textbf{0.0\%} & 50.1\% & 37.7\% & \multicolumn{1}{r|}{11.4\%} & 600   & 8     & 139   & 640 \\
    Sales8  & (50,103) & 21.2\% & \textbf{0.0\%} & \textbf{0.0\%} & \multicolumn{1}{r|}{\textbf{0.0\%}} & 601   & 12    & 55    & 650 \\
    Sales9  & (50,103) & \textbf{0.0\%} & 31.5\% & 6.9\% & \multicolumn{1}{r|}{13.2\%} & 602   & 6     & 202   & 629 \\
    Sales10  & (50,103) & \textbf{0.0\%} & 5.5\% & 5.3\% & \multicolumn{1}{r|}{4.2\%} & 600   & 8     & 137   & 628 \\
    \hline
          & Average   & 18.9\% & 43.9\% & 10.1\% & 5.0\% & 600   & 9     & 132   & 637 \\
    \hline
    \end{tabular}%
      \caption{Performance of core set and benchmark algorithms with ten minutes time limit}
  \label{tab:fat_benchmark}%
\end{table}%

The result for the second experiment is presented in Table \ref{tab:fat_benchmark}. The first two columns describe the datasets, the next four columns present the gap of each algorithm from the best objective value of the four algorithms, and the last four columns report the running time. The smallest gap among the four gaps is in boldface. The stepwise algorithm is the fastest while the gap from the best algorithm is over 40\% on average. The three MIP-based algorithms do not dominate each other: BKM wins six cases, CoreHeur wins six cases, and CoreRnd wins 9 cases. However, the relative gap of CoreRnd is the smallest, which show the effectiveness and robustness of the algorithm given the ten minute time limit. CoreHeur can be a good alternative to CoreRnd because it spends significantly less time than the other two MIP-based algorithms and quickly improves the solution quality of the stepwise algorithm.




\section{Conclusion}

In this study, we present mathematical programs to optimize various subset selection criteria: $MAE$, $MSE$, mRMR, and variants. The proposed mathematical programs return an optimal subset given a valid value of big M, which is also derived in our work. For the selected test instances, we observe that the solver frequently spends more than an hour to prove optimality, while near-optimal solutions are obtained in the first minute. To speed up the solution time and to deal with high dimensional cases, we propose an iterative algorithm based on the popular core set concept. The proposed algorithm and the randomized version converge to local and global optimal solutions, respectively, and show that they outperform the state-of-the-art benchmark. 

Mathematical programming models for subset selection are getting rapidly increasing attention recently due to the improved computational power and numerical solver efficiency. Further, the use of binary decision variables can help to model various subset requirements such as conditional inclusion (exclusion) of explanatory variables. Despite the benefits, there are still limitations in the current mathematical programming models. For example, the current approaches cannot solve large scale instances (e.g., millions of observations or explanatory variables) optimally. Hence, developing an improved model or an efficient algorithm with guaranteed optimality is crucial. Also, the big M values derived in the current work are valid, but not the tightest; this slows down the branch and bound algorithm speed. Hence, tighter big M values can be further studied.

\subsection*{Acknowledgments}
The authors appreciate the editors and reviewers for their constructive comments and suggestions that strengthened the paper.

\bibliographystyle{plain}




\appendix
\section*{APPENDIX}
\label{sec_appendix}

\section{Proof of Lemmas and Propositions}
\label{appendix_proofs_of_lemmas}

\noindent \textbf{Proof of \autoref{proposition_2_and_3_equivalent}}

\noindent The proof is based on the fact that feasible solutions to \eqref{formulation_subset_mae_derive2} and \eqref{formulation_subset_mae_derive3} map to each other. Hence, we consider the following two cases.

\begin{enumerate}
\item Case: \eqref{formulation_subset_mae_derive2} $\Rightarrow$ \eqref{formulation_subset_mae_derive3}\\
Let $S = \{ j | z_j = 1 \}$ be the column index set of a solution to \eqref{formulation_subset_mae_derive2}. We set $v_j = u$ for $j \in S$ and $v_j = 0$ for $j \notin S$. Then,
\vspace{0.1cm}

\begin{tabular}{lll}
$\sum_{i \in I} |t_i|$ & $=(n-1)u- \sum_{j \in J} u z_j$ & (from \eqref{formulation_subset_mae_derive2_b})\\[0.1cm]
					& $= (n-1)u-\sum_{j \in S} u $\\[0.1cm]
					& $= (n-1)u - \sum_{j \in S} v_j$ & (by definition of $v_j$)\\[0.1cm]
					& $ = (n-1)u-\sum_{j \in J} v_j$,
\end{tabular}
\vspace{0.1cm}

\noindent which satisfies \eqref{formulation_subset_mae_derive3}. Further, we satisfy the following.
\begin{enumerate}
\item Constraint \eqref{formulation_subset_mae_derive3_e}: We have $v_j = u \leq u$ for $j \in S$ and $v_j = 0 \leq u$ for $j \notin S$. Hence, $v_j \leq u$ for all $j \in J$.
\item Constraint \eqref{formulation_subset_mae_derive3_f}: We have $u - M(1-z_j) = u \leq v_j = u \leq M z_j = M$ for $j \in S$ and $u-M(1-z_j) = u-M \leq v_j = 0 \leq M z_j = 0$ for $j \notin S$. Hence, we satisfy \eqref{formulation_subset_mae_derive3_f}.
\item Constraint \eqref{formulation_subset_mae_derive3_g}: We have $v_j \in \{0,u\} \geq 0$, for all $j \in J$.
\end{enumerate}
Note that \eqref{formulation_subset_mae_derive3_c} is automatically satisfied since it is equal to \eqref{formulation_subset_mae_derive2_c}. Hence, we obtain a feasible solution to \eqref{formulation_subset_mae_derive3}.
\item Case: \eqref{formulation_subset_mae_derive3} $\Rightarrow$ \eqref{formulation_subset_mae_derive2}\\
Let $S = \{ j | z_j = 1 \}$ be the column index set of a solution to \eqref{formulation_subset_mae_derive3}. Since we are minimizing $u$, \eqref{formulation_subset_mae_derive3_e} is equivalent to $\max_{j} v_j = u$. Note that, in an optimal solution, we must have $v_j = u$ for all $j \in S$. Hence, starting from \eqref{formulation_subset_mae_derive3_b}, we derive
\vspace{0.1cm}

\begin{tabular}{lll}
$\sum_{i \in I} |t_i|$ & $= (n-1)u-\sum_{j \in J} v_j$ & (from \eqref{formulation_subset_mae_derive3_b})\\[0.1cm]
					& $ = (n-1)u-\sum_{j \in S} v_j =(n-1)u-\sum_{j \in S} u$ & ($v_j = u$ for all $j \in S$)\\[0.1cm]
					& $= (n-1)u-\sum_{j \in S} u z_j = (n-1)u-\sum_{j \in J} u z_j$,
\end{tabular}
\vspace{0.1cm}

\noindent which satisfies \eqref{formulation_subset_mae_derive3}.
\end{enumerate}
This ends the proof. $\hfill \square$

\vspace{0.5cm}

\noindent \textbf{Proof of \autoref{proposition_complementary_t}}

\noindent Let $\bar{X} = (\bar{x}, \bar{y}, \bar{v}, \bar{u}, \bar{t}, \bar{z})$ be an optimal solution to \eqref{formulation_subset_mae} and let $\bar{p} = \sum_{j \in J} \bar{z}_j$ be the number of optimal regression variables. For a contradiction, let us assume that there exists an index $k$ such that $\bar{t}_k^+ >0$ and $\bar{t}_k^- >0$. Without loss of generality, let us also assume $\bar{t}_k^+ \geq \bar{t}_k^-$. For simplicity, let $\delta = \bar{t}_k^-$. Let us generate $\tilde{X}$ that is equal to $\bar{X}$ except $\tilde{t}_k^+ = \bar{t}_k^+ - \delta$, $\tilde{t}_k^- = \bar{t}_k^- - \delta = 0$, $\tilde{u} = \bar{u} - \frac{2\delta}{n-1-\bar{p}}$, and $\tilde{v}_j = \tilde{u}$ if $\bar{z}_j = 1$. We show that $\tilde{X}$ is a feasible solution to \eqref{formulation_subset_mae} with strictly lower cost than $\bar{X}$. 
\begin{enumerate}[itemsep=0.01cm]
\item $\tilde{X}$ has lower cost than $\bar{X}$ since $\tilde{u} < \bar{u}$ by definition.
\item $\tilde{X}$ satisfies \eqref{formulation_subset_mae_b} because $\sum_{i \in I} ( \tilde{t}_i^+ + \tilde{t}_i^-)= \sum_{i \in I} ( \bar{t}_i^+ + \bar{t}_i^-) - 2 \delta = (n-1) \bar{u} - \sum_{j \in J} \bar{v_j} - 2 \delta   = (n-1 - \bar{p})\bar{u} - 2 \delta = (n-1 - \bar{p}) (\bar{u} - \frac{2 \delta}{n-1 - \bar{p}})  = (n-1 - \bar{p}) \tilde{u} = (n-1)\tilde{u} - \sum_{j \in J} \tilde{v}_j $, in which the second equality holds because $\bar{X}$ satisfies  \eqref{formulation_subset_mae_b}.

\item Observe that \eqref{formulation_subset_mae_c}, \eqref{formulation_subset_mae_d}, and \eqref{formulation_subset_mae_e} are automatically satisfied. Further, since we set $\tilde{v}_j = \tilde{u}$ for $j$ such that $\tilde{z}_j = 1$, \eqref{formulation_subset_mae_f} and \eqref{formulation_subset_mae_g} are satisfied.
\item Finally, \eqref{formulation_subset_mae_h} is automatically satisfied except for $\tilde{t}_k^+$,$\tilde{t}_k^-$, and $\tilde{u}$. Note that $\tilde{t}_k^+ = \bar{t}_k^+ - \delta =  \bar{t}_k^+ -  \bar{t}_k^- \geq 0$ and $\tilde{t}_k^- = 0$. Also, we have
\vspace{0.1cm}

\begin{tabular}{llll}
$\tilde{u}$ & $=$ & $\bar{u} - \frac{2\delta}{n-1-\bar{p}}$ \quad $=$ \quad $\frac{\sum_{i \in I} ( \tilde{t}_i^+ + \tilde{t}_i^-)}{n-1-\bar{p}} - \frac{2\delta}{n-1-\bar{p}}$\\
	& $=$ & $\frac{\sum_{i \in I \setminus \{k\}} ( \tilde{t}_i^+ + \tilde{t}_i^-) + (\tilde{t}_k^+ + \tilde{t}_k^-) - 2 \delta}{n-1-\bar{p}}$\\
	& $\geq$ & $\frac{\sum_{i \in I \setminus \{k\}} ( \tilde{t}_i^+ + \tilde{t}_i^-) + 2 \tilde{t}_k^- - 2 \delta}{n-1-\bar{p}}$ & (since $\tilde{t}_k^+ \geq \tilde{t}_k^-$)\\
	& $=$ & $\frac{\sum_{i \in I \setminus \{k\}} ( \tilde{t}_i^+ + \tilde{t}_i^-)}{n-1-\bar{p}}$& (by the definition of $\delta$)\\
	& $\geq$ & $0$.\\
\end{tabular}

\vspace{0.1cm}
\noindent Hence, $\tilde{X}$ satisfies \eqref{formulation_subset_mae_h}.
\end{enumerate}

\noindent Hence, $\bar{X}$ is not an optimal solution to \eqref{formulation_subset_mae}, which is a contradiction. $\hfill \square$

\vspace{0.5cm}

\noindent \textbf{Proof of \autoref{proposition_mse_equality_at_opt}}

\noindent Let $\bar{X} = (\bar{x}, \bar{y}, \bar{v}, \bar{u}, \bar{t}, \bar{z})$ be an optimal solution to \eqref{formulation_subset_mse} with $\bar{p} = \sum_{j \in J} \bar{z}_j$. For a contradiction, let us assume that $\bar{X}$ does not satisfy \eqref{formulation_subset_constraint_for_mse} at equality. Let $\delta = (n-1) \bar{u} - \sum_{j \in J} \bar{v}_j - \sum_{i \in I} (\bar{t}_i^+ - \bar{t}_i^-)^2 > 0$. Let us generate $\tilde{X}$ that is equivalent to $\bar{X}$ except that $\tilde{u} = \bar{u} - \frac{2 \delta}{n-1-\bar{p}}$ and $\tilde{v}_j = \tilde{u}$ if $\bar{z}_j = 1$. We first show that $\tilde{u} \geq 0$ since
\vspace{0.1cm}

\begin{tabular}{llll}
$\tilde{u}$ & $=$ & $\frac{\bar{u}(n-1-\bar{p}) - 2 \delta}{n-1-\bar{p}}$ \quad $=$ \quad $\frac{\bar{u}(n-1)-\bar{u}\bar{p} - 2 (n-1) \bar{u} + 2 \sum_{j \in J} \bar{v}_j +2 \sum_{i \in I} (\bar{t}_i^+ - \bar{t}_i^-)^2}{n-1-\bar{p}}$ \\
	& $=$ & $\frac{\sum_{j \in J} \bar{v}_j - \bar{u}(n-1) + +2 \sum_{i \in I} (\bar{t}_i^+ - \bar{t}_i^-)^2}{n-1-\bar{p}} $ \quad $=$ \quad $\frac{\delta}{n-1-\bar{p}} + \frac{\sum_{i \in I} (\bar{t}_i^+ - \bar{t}_i^-)^2}{n-1-\bar{p}}$ \quad $\geq$ \quad $\frac{\delta}{n-1-\bar{p}}$  \quad $\geq$ \quad $0$,\\
\end{tabular}
\vspace{0.1cm}

\noindent in which the second equality is obtained by the definition of $\delta$. For the remaining part, using a similar technique as in the proof of \autoref{proposition_complementary_t}, it can be seen that $\tilde{X}$ is a feasible solution to \eqref{formulation_subset_mse} with strictly lower objective function value than $\bar{X}$. This is a contradiction. $\hfill \square$

\vspace{0.5cm}

\begin{lemma}
\label{lemma_no_zerro_error}
Let $c$ be a vector that has $1$ for $t_i^+$'s and $t_i^-$'s and $0$ for all other variables of \eqref{LP_bigM}. Then, for every extreme ray $r$ in the recession cone of \eqref{LP_bigM}, we must have $c^{\top} r > 0$.
\end{lemma}
\begin{proof}
Suppose that there exists extreme ray $r$ in the recession cone of \eqref{LP_bigM} with $c^{\top} r \leq 0$. Let us consider linear program min $\{ c^{\top} Y |$ \eqref{LP_bigM_a} - \eqref{LP_bigM_e} $\}$. We have two cases.
\begin{enumerate}
\item Suppose that $c^{\top} r < 0$. Note that $\bar{Y} + \delta r$ is feasible for any $\delta \geq 0$ and a feasible solution $\bar{Y}$, since $r$ is extreme ray. Then, $c^{\top} ( \bar{Y} + \delta r) = c^{\top}  \bar{Y} + \delta c^{\top} r$ goes to negative infinity and thus the LP is unbounded from below. However, from the definition of the LP, the objective value is always non-negative. This is a contradiction.
\item Suppose that $c^{\top} r = 0$. This implies that the LP has the optimal objective value of 0. This contradicts \autoref{assumption_no_zero_total_error} since $c^{\top} Y=0$ implies $\sum_{i=1}^n (t_i^+ + t_i^-) = 0$.
\end{enumerate}
By the above two cases, we must have $c^{\top} r > 0$.
\end{proof}

\noindent \textbf{Proof of \autoref{proposition_boundedLP}}

\noindent From \autoref{lemma_no_zerro_error}, we know that there is no extreme rays with non-positive $\sum_{i=1}^n (t_i^+ + t_i^-)$. For the proof of the proposition, let us assume that \eqref{LP_bigM_x_k_plus} is unbounded and thus there is an extreme ray $r$ such that $\bar{c}^{\top} r < 0$, where $\bar{c}$ is the objective vector of objective function of \eqref{LP_bigM_x_k_plus}. Given such extreme ray $r$, we must have $c^{\top} r >0$ by \autoref{lemma_no_zerro_error}, where $c$ is a vector that has 1 for $t_i^*$'s and $t_i^-$'s and 0 for all other variables of \eqref{LP_bigM}. For a feasible solution $\bar{Y}$ to \eqref{LP_bigM_x_k_plus} and any $\delta \geq 0$, $\bar{Y} = Y + \delta r$ is also feasible. Note that $\delta$ must go to infinity for \eqref{LP_bigM_x_k_plus} to be an unbounded LP. However, $\delta c^{\top} r >0$ implies $\sum_{i \in I} (t_i^+ + t_i^-)$ increases as $\delta$ increases. Hence, $\delta$ must be bounded by \eqref{LP_bigM_a}. This implies that $\bar{Y}$ cannot be bounded for any $\delta$. $\hfill \square$

\vspace{0.5cm}

\noindent \textbf{Proof of \autoref{lemma_conversion}}

\noindent With fixed $\bar{z}_j$, we have fixed $\bar{v}_j$ and $\bar{u}$ from \eqref{formulation_subset_mae_f}. Note that, since $\bar{Y}$ has $SSE$ less than or equal to $T_{max}$, we have $(n-1) \bar{u} - \sum_{j \in J} \bar{v}_j = \sum_{i \in I} (t_i^+ + t_i^-) \leq T_{max}$, which satisfies \eqref{LP_bigM_a}. Observe that $v_j$'s and $u$ can be ignored in \eqref{LP_bigM}. Observe also that \eqref{LP_bigM_c} and \eqref{LP_bigM_d} cover \eqref{formulation_subset_mae_d} and \eqref{formulation_subset_mae_e} regardless of $\bar{z}_j$. Finally, \eqref{formulation_subset_mae_c} and \eqref{LP_bigM_b} are the same. Therefore,  $\tilde{Y} = (\bar{x}^+, \bar{x}^-, \bar{y}^+,\bar{y}^-, \bar{t}^+, \bar{t}^-, \hat{M})$ is feasible for \eqref{LP_bigM}. $\hfill \square$

\vspace{0.5cm}

\section{Alternative Approach for Big \texorpdfstring{$M$}{M}}
\label{appendix_big_M_alternatives}

In this section, we derive an approximated value for Big \texorpdfstring{$M$}{M} for \texorpdfstring{$x_j$}{x}'s in (\ref{formulation_MAE_more_m}) and (\ref{formulation_MSE_more_m}).

\begin{algorithm}[ht]
\caption{Estimate-M}        
\label{algo_big_m_estimation}                           
\begin{algorithmic}[1]   
\STATE \textbf{For} $k \in J$
\STATE \hspace{0.5cm} \textbf{For} $s=1,\cdots,30$
\STATE \hspace{1cm} Pick explanatory variable $k$ and $n-3$ explanatory variables randomly and generate new instances with the selected $n-2$ columns and $n$ observations
\STATE \hspace{1cm} Solve \eqref{formulation_SAE_opt} and set $M_k^s \gets x_k^*$
\STATE \hspace{0.5cm} \textbf{End-For}
\STATE \hspace{0.5cm} $\bar{M}_k \gets average(M_k^1,\cdots,M_k^{30})$, $\sigma^{M_k} \gets$ $std$-$dev(M_k^1,\cdots,M_k^{30})$, $\hat{M_k} \gets \bar{M_k} + 1.65 \sigma_{M_k}$
\STATE \textbf{End-For}
\end{algorithmic}
\end{algorithm}

Instead of trying to get a valid value of $M$, we use a statistical approach to get an approximated value of $M$ for $x_j$. In \autoref{algo_big_m_estimation}, we estimate a valid value of $M$ for each $k$. In Steps 2-5, we obtain 30 i.i.d. sample values of $M$ when explanatory variable $k$ is included in the regression model. Then, in Step 6, we obtain the upper tail of the confidence interval. With 95\% confidence, the true valid value of $M$ is less than $\hat{M}$ in Step 6. Hence, we set $M_k := \hat{M}_k$ for $x_k$ in \eqref{formulation_MAE_more_m} and \eqref{formulation_MSE_more_m} for the fat case ($m > n$).

\section{New Objective Function and Modified Formulations for Fat Case \texorpdfstring{$(m \geq n)$}{(m >= n)}}
\label{appendix_new_obj}
Before we derive the objective function, let us temporarily assume $|J| = n-2$ so that any subset $S$ of $J$ automatically satisfies $|S| = p \leq n-2 = |J|$. We will relax this assumption later to consider $|J| > n-2$. Suppose that we want to penalize large $p$ in a way that the best model with $n-2$ explanatory variables is as bad as a regression model with no explanatory variables. Hence, we want the objective function to give the same value for models with $p=0$ and $p=n-2$. With this in mind, we propose \eqref{def_MAE_plus}, which we call the adjusted $MAE$ .

Let us now assume that $SAE$ is near zero when $p = n-2$, which happens often. Then we have $MAE_a = \frac{SAE + \frac{n-2}{n-2}mae_0}{n-1-(n-2)} = SAE + mae_0 \approx mae_0$. Hence, instead of near-zero $MAE$, the new objective has almost the same value as $mae_0$ when $p=n-2$. Recall that $u = MAE$ and $u$ is the objective function in the previous thin case model. Hence, we need to modify the definitions and constraints. First we rewrite constraint \eqref{formulation_subset_mae_b} as $\sum_{i \in I} (t_i^+ + t_i^-) = (n-1)u - \sum_{j \in J}  z_j \Big( u + \frac{mae_0}{n-2}\Big) $. Let $v_j = (u + \frac{mae_0}{n-2})z_j$. Then, \eqref{formulation_subset_mae_f} and \eqref{formulation_subset_mae_g} are modified to
\begin{align}
&v_j \leq u + \frac{mae_0}{n-2} &\label{new_constraint_MAE_f}\\[-3pt]
&u + \frac{mae_0}{n-2} - M(1-z_j) \leq v_j \leq M z_j . &\label{new_constraint_MAE_g}
\end{align}

Finally, we remove the assumption we made ($|J| = n-2$) at the beginning of this section by adding cardinality constraint
\begin{equation}
\label{constraint_cardinality}
\textstyle \sum_{j \in J} z_j \leq n-2
\end{equation}
and obtain the following final formulations,
\begin{center}
$\min \{u | \eqref{formulation_subset_mae_b} - \eqref{formulation_subset_mae_e}, \eqref{formulation_subset_mae_h} \eqref{new_constraint_MAE_f},\eqref{new_constraint_MAE_g}, \eqref{constraint_cardinality}  \},$
\end{center}
which is presented in \eqref{formulation_MAE_more_m}. In fact, without \eqref{constraint_cardinality}, $MAE_a$ cannot be well-defined since it becomes negative for $p > n-1$ and the denominator becomes 0 for $p=n-1$. Observe that \eqref{formulation_MAE_more_m} is an MIP with $2n+4m+3$ variables (including $m$ binary variables) and $n+5m+2$ constraints. Observe also that \eqref{formulation_subset_mae} with the additional constraint \eqref{constraint_cardinality} can be used for the fat case. However, using $n-2$ explanatory variables out of $m$ candidate explanatory variables can lead to an extremely small $SAE$ as we explained at the beginning of this section.

To obtain a valid value of $M$ for $v_j$'s in \eqref{formulation_MAE_more_m}, we can use a similar concept used in \autoref{REG_section_formulation_subset_selection}. In detail, we set 
\begin{equation}
\label{formula_bigM_v_fat_case}
M:= mae_0 + \frac{mae_0}{n-2} = \frac{n-1}{n-2} mae_0
\end{equation}
for $v_j$'s to consider regression models that are better than having no regression variables. Given a heuristic solution with objective function value $mae_a^{heur}$, we can strengthen $M$ by making solutions worse than the heuristic solution infeasible. Hence, we set $M := mae^{heur}_a + \frac{mae_0}{n-2}$ for $v_j$'s in \eqref{new_constraint_MAE_g}.

However, obtaining a valid value of $M$ for $x_j$'s in \eqref{formulation_MAE_more_m} is not trivial. Note that \eqref{bigM_definition}, which we used for the thin case, is not applicable for the fat case because LP \eqref{LP_bigM} can easily be unbounded for the fat case. One valid procedure is to (\rmnum{1}) generate all possible combinations of $n-2$ explanatory variables and all $n$ observations, (\rmnum{2}) compute $M$ for each combination using the procedure in \autoref{REG_subsection_bigM}, and (\rmnum{3}) pick the maximum value out of all possible combinations. However, this is a combinatorial problem. Actually, the computational complexity of this procedure is as much as that of solving \eqref{formulation_SAE_opt} for all possible subsets. Hence, enumerating all possible subsets just to get a valid big M is not tractable. 

Instead, we can use a heuristic approach to obtain a good estimation of the valid value of $M$. In \autoref{appendix_big_M_alternatives}, we propose a statistic-based procedure that ensures a valid value of $M$ with a certain confidence level. This procedure can give an $M$ value that is valid with $95\%$ confidence. However, for the instances considered in this paper, this procedure gives values of $M$ that are too large because many columns can be strongly correlated to each other. Note that a large value of $M$ can cause numerical errors when solving the MIP's. 

Hence, for computational experiment, we use a simple heuristic approach instead. Let us assume that we are given a feasible solution to \eqref{formulation_MAE_more_m} from a heuristic and $x^{heur}_j$'s are the coefficient of the regression model. Then, we set
\begin{equation}
\label{formula_bigM_heur}
M := \max_{j \in J} |x^{heur}_j |.
\end{equation}
Note that we cannot say that \eqref{formula_bigM_heur} is valid or valid with $95\%$ confidence. If we use \eqref{formulation_MAE_more_m} with this $M$, we get a heuristic (even if \eqref{formulation_MAE_more_m} is solved optimally).

Similar to $MAE_a$ in \eqref{def_MAE_plus}, $MSE_a$ can be defined as
\begin{equation}
\label{def_MSE_plus}
MSE_a = \frac{SSE + \frac{p}{n-2}mse_0}{n-1-p},
\end{equation}
where $mse_0 = \frac{\sum_{i \in I}(b_i - \bar{b})^2}{n-1}$ is the mean squared error of an optimal regression model when $p=0$. Next, similar to \eqref{new_constraint_MAE_f} and \eqref{new_constraint_MAE_g}, we define
\begin{align}
&v_j \leq u + \frac{mse_0}{n-2}, &\label{new_constraint_MSE_f}\\
&u + \frac{mse_0}{n-2} - M(1-z_j) \leq v_j \leq M z_j, &\label{new_constraint_MSE_g}
\end{align}
while \eqref{formulation_subset_constraint_for_mse} remains the same. Finally, we obtain
\begin{equation}
\label{formulation_MSE_more_m}
\min \{u | \eqref{formulation_subset_constraint_for_mse},\eqref{formulation_subset_mae_c} - \eqref{formulation_subset_mae_e}, \eqref{formulation_subset_mae_h} \eqref{new_constraint_MSE_f},\eqref{new_constraint_MSE_g},\eqref{constraint_cardinality}  \}
\end{equation}
for the $MSE_a$ objective. Note that \eqref{formulation_MSE_more_m} is mixed integer quadratically constrained program that has $2n+4m+3$ variables and $n+5m+2$ constraints.

For the core set algorithm, similar to \eqref{formulation_core_MAE}, we have
\begin{equation}
\label{formulation_core_MSE}
\min \{u | \eqref{formulation_subset_constraint_for_mse},\eqref{formulation_subset_mae_c} - \eqref{formulation_subset_mae_e}, \eqref{formulation_subset_mae_h}, \eqref{new_constraint_MSE_f},\eqref{new_constraint_MSE_g} \}.
\end{equation}

\end{document}